\documentclass[a4paper,twoside,12pt,english]{article}

\usepackage[applemac]{inputenc}
\usepackage[english]{babel}
\makeatletter
\renewcommand\section{\@startsection
{section}{1}{0mm}%
{-2\bigskipamount}%
{\bigskipamount}%
{\normalfont\normalsize\bfseries}%
}
\renewcommand\thesection{\arabic{section}}

\newcommand\dateymd{\number\year, \ifcase\month\or
January\or February\or March\or April\or May\or June\or
July\or August\or September\or October\or November\or
December\fi, \number\day}
\newcommand\printtime{%
\c@hours=\time \divide\c@hours by60
\c@minutes=\c@hours \multiply\c@minutes by-60
\advance \c@minutes by \time
\ifnum\c@hours<10 0\fi\the\c@hours:%
\ifnum\c@minutes<10 0\fi\the\c@minutes}
\makeatother

\usepackage[OT1]{fontenc}

\usepackage{fancyhdr}
\usepackage{relsize}
\usepackage{bm} 
\usepackage{amssymb}
\usepackage{eucal}
\usepackage{amsmath}
\usepackage{amsthm}
\usepackage{enumerate}
\usepackage{IEEEtrantools}

\usepackage{verbatim} 

\usepackage[nodvipsnames]{color}

\usepackage[pdftex]{graphicx} 

\newcommand\upla{}
\newcommand\uplapar{}
\newcommand\xtra{} 

\newcommand\bnou{\textup{\textbf{\lower3.7pt\hbox{\char'052}:~}}}
\newcommand\enou{\unskip\textup{\textbf{~:\lower3.7pt\hbox{\char'052}}} }
\newcommand\bvell{\textup{\textbf{\lower3.7pt\hbox{\char'052}:~}$\langle$}}
\newcommand\evell{\unskip\textup{$\rangle$\textbf{~:\lower3.7pt\hbox{\char'052}}} }
\newcommand\bbnou{\textup{\textbf{\lower3.7pt\hbox{\char'052\char'052}:~}}}
\newcommand\eenou{\unskip\textup{\textbf{~:\lower3.7pt\hbox{\char'052\char'052}}} }

\newcommand\ie{i.\,e.~}

\newcommand\ifoi{\,\hbox{if\kern2.5pt and\kern2.5pt only\kern2.5pt if}\,{} }

\newcommand\df{\bfseries}
\newcommand\dfc[1]{\,{\df#1}\,}
\newcommand\dfd[1]{\,{\df#1}\hskip1pt}
\newcommand\secpar[1]{\S\,{#1}}

\newcommand\ensep{\unskip\hskip.65em\ignorespaces}

\newcommand\atilde{\lower3.5pt\hbox{\~{}}}
\newcommand\underl{\lower3.5pt\hbox{-}}

\newcommand\halfsmallskip{\vskip0.5\smallskipamount}

\newcommand\remarks{\medskip\noindent\textit{Remarks}\par}
\newcommand\remark{\vskip-10pt\bigskip\noindent\textit{Remark}.\hskip.5em}

\newcommand\pq[2]{\raise.25ex\hbox{\footnotesize${#1}\over{#2}$}%
\hskip-.35ex\null}
\newcommand\onehalf{\frac12}

\newcommand\itm{\smallskip\noindent\hbox to\parindent{\hss\smaller{$\bullet$}\hskip.5em}}

\newcommand\xxxx[1]{%
 \hangindent2.5\parindent
 \hangafter1
 \noindent\hskip.5\parindent
 \hbox to2\parindent{\hss#1\hss}}

\newcommand\iim[1]{\xxxx{\textup{\small(#1)}}\ignorespaces} 

\newcommand\ddd[1]{\halfsmallskip\vskip-2pt\noindent\hbox to 2\parindent{\hss\footnotesize$\bullet$\ \ }{#1}\ensep}

\newcommand\brwrap[1]{[\textsl{#1}\kern1pt]}

\makeatletter
\newcommand\bibref[1]{\@nameuse{b@#1}}
\renewcommand\@biblabel[1]{\brwrap{#1}}
\renewcommand\@cite[2]{\hbox{\brwrap{#1\if@tempswa\/\upshape\,:\,{\relscale{0.95}#2}\fi}}} 
\makeatother

\newcommand*\dbibref[2]{\bibref{#1}\,\textup{:\,{\relscale{0.95}#2}}}
\newcommand*\refco{\/\kern.1ex\textup{,}\hskip.45ex}
\newcommand*\refsc{\/\kern.15ex\textup{;} }

\newcommand\ustrut{\rule[-.2ex]{0pt}{1.6ex}}
\newcommand\latop[2]{{\scriptstyle\ustrut#1\atop\scriptstyle\ustrut#2}}

\newcommand\sbset{\subset}

\newcommand\sbseteq{\subseteq}

\newcommand\cd[1]{\!#1\!}

\newcommand\thesis{t}

\newtheorem{proposition}{Proposition}[section]
\newtheorem{lemma}[proposition]{Lemma}
\newtheorem{theorem}[proposition]{Theorem}
\newtheorem{corollary}[proposition]{Corollary}

\newcommand\nt[1]{\overline{#1}}		

\newcommand\pset{\varPi_+}		
\newcommand\piset{\varPi}			
\newcommand\spiset{\varSigma}		
\newcommand\lit{p}		    		
\newcommand\liit{q}		    		
\newcommand\liiit{r}		    		
\newcommand\livt{s}		    		
\newcommand\lxt{\alpha}		    
\newcommand\clau{C}				
\newcommand\doct{{\mathcal D}} 	
\newcommand\cnf{\Phi} 			

\newcommand\ist{A}				
\newcommand\xst{X}				
\newcommand\releq{E}				

\newcommand\val{w}				
\newcommand\orv{v}				
\newcommand\orvk{v^k}				
\newcommand\pesk{a_k}				
\newcommand\utv{v'}				
\newcommand\urv{v^*}				
\newcommand\orvbis{w}				
\newcommand\utvbis{w'}			
\newcommand\urvbis{w^*}			

\newcommand\valt{w'}
\newcommand\valtz{\widetilde w{}^{\kern.5pt\prime}}

\newcommand\mg{\eta}				
\newcommand\mgo{\eta_0}

\newcommand\orvz{\widetilde v}	
\newcommand\utvz{\widetilde v{}^{\kern.5pt\prime}}
\newcommand\urvz{\widetilde v{}^{\kern.5pt\ast}}
\newcommand\uurvz{\widetilde u{}^{\kern.5pt\ast}}

\newcommand\ntv[1]{v^{(#1)}}		
\newcommand\ntvz[1]{\widetilde v{}^{(#1)}}	 

\newcommand\doctz{\widetilde{\hbox{\vphantom{t}\smash{$\doct$}}}} 	

\newcommand\clauiii{C_\ast}
\newcommand\clauzi{\widetilde{\hbox{\vphantom{t}\smash{$\clau$}}}\kern-1.5pt{}_1}
\newcommand\clauzii{\widetilde{\hbox{\vphantom{t}\smash{$\clau$}}}\kern-1.5pt{}_2}
\newcommand\clauzzi{\widetilde{\clau}\kern-1pt{}_1}
\newcommand\clauzzii{\widetilde{\clau}\kern-1pt{}_2}

\newcommand\Phiz{\widetilde{\hbox{\vphantom{t}\smash{$\Phi$}}}}

\newcommand\uttvz{\widetilde v{}^{\kern.5pt\prime\prime}}

\newcommand\magn{u}
\newcommand\magnz{\widetilde u}
\newcommand\ptit{\delta}

\newcommand\ltv{{}^\prime\kern-.25pt v}
\newcommand\lrv{{}^\ast\kern-.25pt v}

\newcommand\res[1]{\,\vtop{\offinterlineskip\halign{\hfil##\hfil\cr$\vee$\cr\noalign{\vskip2pt}$\scriptstyle#1$\cr}}\,}
\newcommand\ress[1]{\vtop{\offinterlineskip\halign{\hfil##\hfil\cr$\vee$\cr\noalign{\vskip2pt}$\scriptscriptstyle#1$\cr}}}
\renewcommand\res[1]{\mathbin{\vtop{\baselineskip0pt\lineskip.3ex\halign{\hfil##\hfil\cr$\vee$\cr$\scriptstyle{#1}$\cr}}}}
\renewcommand\ress[1]{\mathbin{\vtop{\baselineskip0pt\lineskip.25ex\halign{\hfil##\hfil\cr$\scriptstyle\vee$\cr$\scriptscriptstyle{#1}$\cr}}}}

\begin{document}


\thispagestyle{empty}

\renewcommand\footnoterule{\rule{10mm}{0pt}}

\null\vskip-24mm\null 

\begin{center}
\hrule
\vskip7.5mm
\textbf{\uppercase{A general method for deciding}}\linebreak
\textbf{\uppercase{about logically constrained issues}\,%
\footnote{This work stemmed from our attendance at the \textsl{MOVE\,-\,Urrutia Elejalde} workshop on \textsl{Judgement Aggregation} that was held in Barcelona on December 14th–16th, 2009.\ensep
Our research into certain particular aspects has been motivated also by the attendance of two of us at the workshop \textsl{New Developments in Judgement Aggregation and Voting Theory}, held in Lauterbad, Germany, on~September 9th--11th, 2011.
\ensep
We are grateful to the organizers of both these workshops for their kind invitations.}
}
\par\medskip
\textsc{Rosa Camps,\, Xavier Mora \textup{and} Laia Saumell}
\par
Departament de Matem\`{a}tiques,
Universitat Aut\`onoma de Barcelona,
Catalonia,
Spain
\par\medskip
\texttt{xmora\,@\,mat.uab.cat}
\par\medskip
July 15, 2010;\ensep revised March 3, 2012
\vskip5mm
\hrule
\end{center}

\null\vskip-7.5mm\null
\begin{abstract}
\noindent
A general method is given for
revising degrees of belief
and arriving at consistent decisions
about a system of logically constrained issues.
\linebreak[3]\ensep
In contrast to other works about belief revision,
here the constraints are assumed to be fixed.
\ensep
The method has two variants, dual of each other,
whose revised degrees of belief are respectively above and below the original ones.
\ensep
The~upper [resp.~lower] revised degrees of belief are uniquely characterized
as~the lowest [resp.~greatest] ones that are invariant by a~certain max-min
[resp.~min-max] operation
determined by the logical constraints.
\ensep
In~both variants, making balance between the revised degree of belief of a proposition
and that of its negation leads to decisions that are ensured to be consistent with the logical constraints.
\ensep
These decisions are ensured to agree with the majority criterion
as applied to the original degrees of belief  
whenever this gives a consistent result.
\ensep
They are also ensured to satisfy
a property of respect for unanimity about any particular issue,
as well as a property of monotonicity with respect to the original degrees of belief.
\ensep
The application of the method to certain special domains
comes down to well established or increasingly accepted methods,
such as the single-link method of cluster analysis
and the method of paths in preferential voting.


\bigskip\noindent
\textbf{Classification MSC2010:}\hskip.75em
\textit{%
03B42, 
68T37, 
91B06, 
91B14, 
91C20. 
}
\end{abstract}

\vskip2.5mm
\hrule

\pagebreak 

\null\vskip-20mm\null 
\section{Introduction}

\renewcommand\uplapar{\vskip-7mm\null}

\uplapar
\paragraph{1.1} 
Around 1990 Lewis A.~Kornhauser and Lawrence G.~Sager pointed out that
collegial courts are liable to what they termed the \dfc{doctrinal paradox} 
\cite{kii, ks93}.
A~simple example of it would be the following:
A~person is on trial for having committed a crime.
The case involves two issues $p$ and $q$ whose conjunction, \ie both of them being true,
determines whether the accused is guilty or not.
The case is heard by a jury of three members.
One of them believes that $p$ is true but not $q$;
 accordingly, he finds the accused not guilty.
Another one believes that $q$ is true but not $p$;
 so he also finds the accused not guilty.
Finally, the third member of the jury believes that both $p$ and $q$ are true,\linebreak[3]
 so~he finds the accused guilty.
Altogether, one can say that the jury has reached a majority verdict of not guilty.
However, one can also say that they have a majority opinion that $p$ is true 
and that $q$ is also true; accordingly, the accused should be considered guilty.
The above-mentioned authors acknowledge that 
``We have no clear understanding of how a court should proceed in cases where the doctrinal paradox arises''~\cite{ks93}.

The main issue in a trial is whether the accused is guilty or not. Let $\thesis$ 
denote the proposition that he is guilty. We are assuming that this proposition is logically connected to $p$ and $q$ as specified by the following \dfd{doctrine\mdseries:} $\thesis\leftrightarrow p\land q$. 
Every member of the jury is required to be consistent with it. Therefore, there are only four consistent opinions about the truth of $(p,q,t)$, namely:
$(1,1,1)$, $(1,0,0)$, $(0,1,0)$ and $(0,0,0)$,
where $1$ means true, and $0$~means false.
Let us consider all possibilities for a jury that is hearing such a case:
let $x,y,z,u$ be the fractions who adhere to, respectively, each of those four consistent opinions. 
In terms of these numbers, the fractions of the jury who believe in the truth of $t,p,q$ are res\-pectively $v_\thesis = x$, $\orv_{p} = x+y$ and $\orv_{q} = x+z$.
These numbers can be seen as degrees of collective belief in the truth of the respective propositions.
A natural criterion for col\-lectively deciding about $\thesis$ is to consider it true whenever $v_t > \onehalf$, \ie $x > \onehalf$.
In the following we will refer to it as the \dfc{conclusion-based criterion}.
In contrast, the \dfc{premise-based criterion}
con\-siders $\thesis$ true \ifoi both $\orv_{p},\orv_{q} > \onehalf$, \ie both $x\cd+y,\,x\cd+z>\onehalf$.
Clearly, if~$\thesis$~is found true by the conclusion-based criterion, then it will also be found true by the premise-based one. However, the converse does not hold,
as it is exemplified in the preceding paragraph, where $\orv_{p}=\orv_{q}=\frac23 > \onehalf$ but $v_\thesis = \frac13 < \onehalf$.

\renewcommand\uplapar{\vskip-9mm\null}

\uplapar
\paragraph{1.2}
The core of the problem is that the majority rule does not keep consistency with the doctrine.
Even though each individual votes in a consistent way, the outcome of the majority rule need not be consistent!
By the \dfc{majority rule} we mean accepting a proposition $\lxt$ and rejecting its negation $\nt\lxt$
whenever $v_\lxt > \onehalf$.

From this point of view, the problem is entirely analogous to the well-known paradox pointed out in the eighteenth century by the Marquis of Condorcet in connection with preferential voting \cite{mu}: When several\linebreak[3] individuals vote on three or more alternatives by ordering them according to their preferences, the majority rule may result in a cyclic (\ie non transitive) binary relation.
\ensep
In preferential voting one is interested in the propositions\linebreak[3]
$p_{xy}$\,: `$x$~is preferable to~$y$', where $x$ and $y$ vary over all possible pairs of alternatives, and the doctrine is transitivity, namely $p_{xy} \land p_{yz} \rightarrow p_{xz}$ for any three alternatives $x,y,z$, 
together with asymmetry, namely $p_{xy} \rightarrow \nt p_{yx}$ for any two alternatives $x,y$.
As before, the problem is that the majority rule does not keep consistency with the doctrine.
The standard example involves three alternatives $a,b,c$ and three rankings, namely $a\succ b\succ c$, \,$b\succ c\succ a$ and $c\succ a\succ b$. If~$v_{xy}$ denotes the fraction of times that $x$~is preferred to~$y$, one gets $v_{ab} = v_{bc} = v_{ca} = \frac23$ and $v_{ba} = v_{cb} = v_{ac} = \frac13$. Clearly, the condition $v_{xy}>\onehalf$ does not define a transitive relation.

Yet another class of objects whose aggregation and subsequent application of the majority rule may break away from the corresponding  doctrine are equivalence relations. Consider, for example, the set $\{a,b,c\}$ and the equivalence relations asso\-ciated respectively with the three following partitions: $\{\{a,b\},\{c\}\}$, $\{\{a\},\{b,c\}\}$ and $\{\{a,b,c\}\}$. If $v_{xy}$ denotes the fraction of times that $x$~and~$y$ belong to the same class, one gets $v_{ab} = v_{bc} = \frac23$ and $v_{ac} = \frac13$ (together with~$v_{yx} = v_{xy}$). Here too, the aggregation operation does away with transitivity.

All these problems are particular cases of a more general one
where the objects being aggregated are systems of degrees of belief 
for several propositions which are logically constrained by a certain doctrine.
This general problem can be referred to as that of
\dfc{constrained judgment aggregation.}
\ensep
In~a~celebrated paper published in 1952,
Georges Th.~Guilbaud already 
identified this problem as a generalization of Condorcet's one:
``The~general logic of propositions ... teaches us that the problem is universal. Given several propositions or questions, every logical relation between them can be expressed by establishing the list of possible arrangements of signs and the list of impossible arrangements. ... the rule of the majority may very well lead to a forbidden arrangement''~\cite{gui}.



\uplapar
\paragraph{1.3} 
Condorcet's paradox is closely related to the celebrated impossibility theorem formulated in 1950--63 by Kenneth J.~Arrow
\cite{arrow, geana}.
This theorem is concerned with preferences expressed by means of complete rankings, ties allowed, and with rules for (deterministically) aggregating any given set of individual preferences of this form into a collective one of the same form. 
Quite naturally, one would be interested in rules that comply with the following conditions:\ensep
(i)~Anonymity:
all individuals play the same role;\ensep
(ii)~Respect for unanimity:
if~every individual strictly prefers $x$ to $y$, 
then the collective ranking strictly prefers also $x$ to $y$;\ensep 
(iii)~Independence of irrelevant alternatives:\linebreak[3]
the collective preference about two alternatives $x$ and $y$ depends only on the individual preferences about $x$ and $y$.\ensep
According to Arrow's theorem, such a rule does not exist, 
except in the case where there are only two alternatives.


The doctrinal paradox has motivated the question of extending Arrow's theorem to the general problem of constrained judgment aggregation.
Analogously to the preceding paragraph, one would be interested in rules that,\linebreak[3]
besides keeping consistency with the doctrine in question,
comply also with the following conditions:\ensep
(i)~Anonymity;\ensep
(ii)~Respect for unanimity:
if a particular proposition is accepted by every individual, then it is also accepted by 
the collective judgment;\ensep
(iii)~Issue-by-issue aggregation (or independence):
the collective judgment about each issue depends only on the individual judgments about it.\ensep
In~the present context, the possibility or impossibility of a rule satisfying these conditions clearly depends on the structure of the doctrine under consideration. For instance, the majority rule will always do for an empty doctrine, \ie several propositions without any logical connection between them.\linebreak[3]
Accordingly, the existing impossibility results, for which we refer to~\cite{np08, np10, dl10b, lipo, dl}, 
specify certain conditions to be met by the doctrine.


\uplapar
\paragraph{1.4}
In practice, one is bound to make decisions, 
even for doctrines that are included in the above-mentioned impossibility results.
The purpose of this article is to put forward a general rule for making such decisions
in consistency with the doctrine.
This rule will not comply with the property of issue-by-issue aggregation.
However, it will be anonymous and it will respect unanimity about any particular issue
whenever the individual beliefs are consistent with the doctrine.


The proposed method generalizes two other ones that are already known in certain particular areas: In fact, when the doctrine corresponds to the notion of an equivalence relation on a certain set of items, then one of the variants of our method yields the so-called single-link method of cluster analysis~\cite{js}. On~the other hand, when the doctrine corresponds to the notion of a total order, then the proposed method corresponds essentially to the one that was introduced in 1997 by Markus Schulze
\hbox{\brwrap{
\bibref{sc}\refco
\bibref{scbis}\refsc
\dbibref{t6}{p.\,228--232}\refsc%
\bibref{crc}\refco
\bibref{cri}%
}}.


The problem considered in this article
should be distinguished from a related but different one that
is often referred to as `belief revision'.
This subject is reviewed for instance in \cite{hansson}.
In contrast to the present work, where the doctrine remains fixed,
there one allows for the possibility of revising~it.
\halfsmallskip\pagebreak 

The remainder of this article is organized in five more sections: 
\ensep
2:~Setting up the problem.
\ensep
3:~Construction and main results.
\ensep
4:~Technical issues and supplementary topics.
\ensep
5:~Application to specific domains.
\ensep
6:~Discussion and interpretation of the results.


\section{Setting up the problem}


In this section we introduce in more detail the objects that we will be dealing with.
We will be definitely based upon propositional logic.
However, instead of dealing with all-or-none truth assignments,
we will deal with degrees of belief in the whole range from $0$ to $1$.
The term `belief' should be understood here in a very wide sense;
depending on the context, it may be more appropriate to use other terms,
such as `plausibility', `certainty',\linebreak[3] `evidence', et cetera.
On~the other hand, we will also deal
with decision values within a set of three alternatives meaning respectively `accepted',\linebreak[3] `rejected' and `undecided'.
Allowing for undecidedness is unavoidable as soon as one looks at aggregating different views on the same issues.
Besides, the notion of aggregation also calls for allowing
the degree of belief of a negation $\nt p$ to be independent from that of~$p$.
This leads to viewing $p$ and $\nt p$ as antagonistic to each other but not necessarily mutually exclusive.


\paragraph{2.1}
To begin with, we~are interested in a finite
set of basic propositions $p,q,r,...$ together with their respective negations $\nt p,\nt q,\nt r,...$.
Following the standard terminology of logic,
we will refer to $p,q,r,...$ as \hbox{\dfd{atoms},} 
and an atom or its negation will be called a \dfd{literal}.
The set of atoms 
will be denoted as $\pset$,
and the set of literals will be denoted as~$\piset$.
So, $\piset = \cup_{p\in\pset}\,\{p,\nt p\}$.
\ensep
A \dfc{truth assignment} is a mapping whereby each literal is assigned one of the values  `true' or `false', with the restriction that $\nt\lit$~is false [resp.~true] whenever $\lit$~is true [resp.~false].

We will also deal with compound propositions. They are represented by \dfc{formulas} that combine atoms by means of the Boolean operators of propositional calculus, such as $\lnot,\land,\lor,\rightarrow$ and $\leftrightarrow$.
\ensep
We are assuming that $\lnot p=\nt p$ and $\lnot(\nt p) = p$ 
for any $p\in\pset$.
\ensep
The notions of \dfc{entailment} \hbox{---or }logical implication--- and \dfc{logical equivalence} between formulas will be understood exactly as in classical bivalent propositional logic:
One formula entails another if there is no truth assignment that makes the first one true and the second one false according to the rules of propositional calculus.
Two formulas are logically equivalent to each other when their truth value is the same for any truth assignment.

A~\dfc{doctrine} can be seen as a compound proposition ---in other words, a~formula---
whose truth is assumed to hold.
A truth assignment that makes this formula true will be
said to be
\dfc{consistent} with the doctrine.

In dealing with the doctrine, we will make a crucial use of the well-known fact that any formula can be transformed into logically equivalent ones of the form
\begin{equation}
\label{eq:cnf}
\cnf(\doct) \,:=\,
\bigwedge_{\clau\in\doct} \left(\,\bigvee_{\lit\in\clau} \lit\right),
\end{equation}
where $\doct$ stands for a collection of subsets of $\piset$.
A formula of this form is called a \dfc{conjunc\-tive normal form},
and the expressions within parentheses are called its \dfd{clauses}.
Generally speaking, a clause means any formula of the form
\begin{equation}
\label{eq:clause}
\phi(\clau) \,:=\,
\bigvee_{\lit\in\clau} \lit,
\end{equation}
where $\clau$ is a subset of $\piset$.
Obviously, specifying a clause is equivalent to specifying the associated set $\clau\sbset\piset$,
and specifying a conjunctive normal form is equivalent to specifying the associated collection $\doct$ of subsets of~$\piset$.
Because of that, in the following we will sometimes refer to the sets $\clau\in\doct$ themselves as `clauses' and we may even refer to the collection $\doct$ as the `doctrine'.

A conjunctive normal form being true means that each of its clauses is true.
On the other hand, a clause being true means that at least one of its literals is true;
in other words, if all of its literals but one are known to be false, then the remaining one must be true. Therefore, the doctrine associated with (\ref{eq:cnf})
provides the following implications:
\begin{equation}
\label{eq:pimplicant}
\lit \,\leftarrow\, \bigwedge_{\latop{\lxt\in\clau}{\lxt\neq\lit}} \nt\lxt,
\end{equation}
for any $\clau\in\doct$ such that $\lit\in\clau$.
The method developed below will be based crucially on such implications.

\smallskip
The conjunctive normal forms equivalent to a given doctrine are by no means unique.
Later on, we will
restrict ourselves 
to the special class of
\dfd{prime conjunctive normal forms},
and eventually we will choose a particular member of that class
which is known as the \dfc{Blake canonical form.}
These concepts will be introduced in \secpar{4.1}.

\medskip
On the other hand, we adopt from now on the following assumptions:

\halfsmallskip
\iim{D1}The doctrine is satisfiable
\hskip1.25pt and \hskip1.25pt
it is given in conjunctive normal form.

\halfsmallskip
\iim{\small D2}It does not contain unit clauses, \ie clauses with a single literal.
\ensep
Otherwise, one can always fix it by replacing such literals and their negations by the corresponding truth values, and deleting them
from~$\piset$.

\halfsmallskip
\iim{D3}It explicitly contains each of the \textit{tertium non datur} clauses $\lit\lor \nt\lit$,\linebreak 
with $\lit\in\pset$. This special convention provides the trivial implications
$\lit\leftarrow\lit$ for any $\lit\in\piset$.


\paragraph{2.2}
We call~\dfc{valuation} any mapping $\val$ whereby each literal $\lit$ in $\piset$ is assigned a~value $\val_\lit$ in the interval of real numbers from~$0$ to~$1$.
The number~$\val_\lit$, sometimes denoted alternatively as $\val(\lit)$, will be seen as a degree of belief in the proposition~$\lit$.

Generally speaking, the values of $\val_\lit$ and $\val_{\nt\lit}$
need not add up to $1$,\linebreak[3]
but their sum can take any value from $0$ to $2$.
Having $\val_\lit+\val_{\nt\lit} < 1$ means a~lack of information,
whereas $\val_\lit+\val_{\nt\lit} > 1$ means that some contradiction is present.
A valuation that satisfies $\val_\lit+\val_{\nt\lit}=1$ for any $\lit\in\piset$ will be called \dfd{balanced}.


Later on, we will sometimes compare two valuations $\orv$ and~$\val$. In that connection, $\orv\le\val$ will mean simply that $\orv_\lit\le\val_\lit$ for any $\lit\in\piset$.

\medskip
The truth assignments of classical propositional logic can be seen as balanced valuations that take only the values $0$ and $1$, where $1$ means `true' and $0$ means `false'.
Since they are formally the same, in the sequel we will often identify truth assignments with balanced all-or-none beliefs.

\medskip
We will also be interested in balanced valuations with values in $\{0,\onehalf,1\}$,
where the value $\onehalf$ can be interpreted as `undefined' or `undecided'.
We will refer to such valuations as \dfd{partial truth assignments},
a term that we are taking from satisfiability theory.
They will be used
mainly for specifying decisions,
in which case the values $1$, $0$ and $\onehalf$ can be interpreted 
as meaning respectively `accepted', `rejected' and `undecided'.

A partial truth assignment $u$ will be said to be \dfc{definitely consistent} with~$\doct$,
or with $\Phi(\doct)$, 
when, for each clause $\clau\in\doct$ and every $\lit\in\clau$,\linebreak[3]
the following implication holds:\ensep
if $u_\lxt=0$ for every $\lxt\in\clau\setminus\{\lit\}$,
then $u_\lit=1$.
When no undecidedness is present, a partial truth assignment being definitely consistent is equivalent to its being consistent as a truth assignment. 
When undecidedness is allowed, definite consistency requires every clause to contain at least one accepted literal, or alternatively, at least two undecided literals.


\pagebreak
Every valuation $\val$ gives rise to a (partial) decision in the following way,
that depends on a parameter $\mg$ in the interval $0\le\mg\le1$:
For any $\lit\in\piset$,
\begin{alignat}{4}
&\text{$\lit$ is accepted and $\nt\lit$ is rejected} &&\quad\text{whenever}\ \ 
\hphantom{|\,}\val_\lit &&- \val_{\nt\lit}\hphantom{\,|} &&\,>\, \mg,
\label{eq:acceptedrejected}
\\
&\text{$\lit$ and $\nt\lit$ are left undecided} &&\quad\text{whenever}\ \ 
|\,\val_\lit &&- \val_{\nt\lit}\,| &&\,\le\, \mg.
\label{eq:undecided}
\end{alignat}
\xtra
We will refer to it as the \dfc{decision of margin $\mg$} associated with~$\val$, and
we will identify it with the corresponding partial truth assignment,
which will be denoted as $\mu_\mg(\val)$.
In the case $\mg=0$ we will call it the \dfc{basic decision} associated with $\val$,
and the corresponding partial truth assignment will be denoted as $\mu(\val)$.
\ensep
In tune with these definitions, the difference $\val_\lit-\val_{\nt\lit}$ will be called the \dfc{acceptability} of~$\lit$ according to $\val$.
\ensep
If the valuation $\val$ is balanced, then the basic decision criterion is equivalent to the majority rule of §{1.2}, 
namely accepting $\lit$ and rejecting $\nt\lit$ whenever $\val_\lit>\onehalf$.

\paragraph{2.3}
Valuations are often an aggregate of several components (members of a jury, decision criteria, etc.), \ie an average of the form
\begin{equation}
\label{eq:convcomb}
\orv_\lit \,=\, \sum_k \pesk\,\orvk_\lit,
\end{equation}
where $\orvk$ are the component valuations,
and $\pesk$ are the corresponding relative frequencies or weights,
satisfying $\pesk\ge0$ for any $k$ and $\sum_k\pesk=1$.
If the component valuations have an all-or-none character,
then $\orv_\lit$ is the fraction of components where $\lit$ is considered valid.

The doctrinal paradox points out the possibility that the average valuation be not consistent with the doctrine even when all component valuations are consistent with it.

\paragraph{2.4}
Our main aim can be stated in the following way: \textit{Given a~valuation~$\orv$,
build a revised one that
achieves consistency 
with the doctrine
while staying as near as possible to $\orv$.}

In the method presented below, the consistency of the revised valuation will hold in a general sense that entails the definite consistency of the associated decisions in the sense of the definition given in~\secpar{2.2}.

Generally speaking, we will obtain two revised valuations $\urv$ and $\lrv$\linebreak[3] satisfying the inequalities $\lrv\le\orv\le\urv$.
We will call them respectively the upper and lower revised valuations.
\ensep
For a balanced original valuation, the upper and lower revised valuations give rise to exactly the same decision. However, in the unbalanced case they can lead to different decisions.

For the moment, we will be concerned only with the upper revised valuation $\urv$, 
the lower one being introduced at the end by duality.

\section{Construction and main results}

\paragraph{3.1}
\textbf{Construction.}\ensep
The upper revised valuation will be obtained by means of an iterative process
whereby belief will be propagated along the implications contained in the doctrine.
More specifically, we will consider the implications (\ref{eq:pimplicant})
and we will apply the following general principle:

\newcommand\teof{P}
\halfsmallskip
\iim{\teof}%
Consider an
implication of the form $\lit\leftarrow \bigwedge_{\lxt\in\spiset}\lxt$
with $\spiset\sbset\piset$. 
As~soon as the right-hand side is satisfiable, this implication
gives to~$\lit$ at least the same degree of belief
as the weakest of the conjuncts~$\lxt$.
\halfsmallskip

The hypothesis of satisfiability contained in this principle should be understood
in relation to the doctrine under consideration:
Among the truth assignments that make the doctrine true, 
at least one of them should make the conjunction $\bigwedge_{\lxt\in\spiset}\lxt$ true.
Unless otherwise stated, in the sequel `satisfiability' should always be understood in this sense.
\ensep
In \secpar{4.1} we will see that this
amounts to requiring the conjunctive normal form
that expresses the doctrine
to be a prime conjunctive normal form.
However, for the moment we need not be concerned with this question.
In fact, principle~{\small(\teof)} will be used only to motivate certain definitions
that can be applied to any conjunctive normal form.

In contrast, the present subsection already makes use of the assumptions {\small(D1, D2, D3)} adopted at the end of \secpar{2.1}.

\bigskip 
Let us look at the consequences of applying principle~{\small(\teof)} to the implications~(\ref{eq:pimplicant}).
Starting with the degrees of belief given by $\orv$, we infer that every $\lit\in\piset$ should be believed at least in the degree $\utv_\lit$ defined by
\begin{equation}
\label{eq:vprime}
\utv_\lit \,=\, \max_{\latop{\clau\in\doct}{\clau\ni\lit}}\, \min_{\latop{\lxt\in\clau}{\lxt\neq\lit}} \,\orv_{\nt\lxt},
\end{equation}
where the $\max$ and $\min$ operators are ensured to deal with non-empty sets of values as a consequence of assumptions {\small(D2)} and {\small(D3)}.


\medskip
\begin{lemma}\hskip.5em
\label{st:step}
The transformation $\orv\mapsto\utv$ has the following properties:

\iim{a}It is continuous.

\iim{b}$\orv \le \orvbis$ implies $\utv \le \utvbis$.

\iim{c}$\orv \le \utv$.

\iim{d}The image set of $\utv$ is contained in that of $\orv$.
\end{lemma}

\begin{proof}\hskip.5em
Part~(a): Since $\max$ and $\min$ are continuous.
Part~(b): Since $\max$ and $\min$ are monotone.
Part~(c): Because of the \textit{tertium non datur} clauses provided by {\small(D3)}, one of the arguments of the $\max$ operator in the right-hand side of~(\ref{eq:vprime}) is $\orv_\lit$. 
Part~(d): Formula~(\ref{eq:vprime}) entails that $v'_\lit$ coincides with $\orv_{\nt\lxt}$ for some $\lxt\in\piset$.
\end{proof}

\medskip
As soon as we accept $\utv$ as new degrees of belief, it makes sense to repeat the same operation with $\orv$ replaced by $\utv$, thus obtaining a still higher valuation~$\orv''$, and so on. By proceeding in this way, one obtains a non-decreasing sequence of valuations $\ntv{n}\ (n = 0,1,2,\dots)$ with the property that all of them take values in the same finite set, namely  $\{\,\orv_\lxt\mid\lxt\in\piset\,\}$. Obviously, this implies that this sequence will eventually reach an invariant state~$\urv$. This eventual valuation is, by definition, the \dfd{upper revised valuation}.


\medskip
\begin{theorem}\hskip.5em
\label{st:rev}
The transformation $\orv\mapsto\urv$ has the following properties:

\iim{a}It is continuous.

\iim{b}$\orv \le \orvbis$ implies $\urv \le \urvbis$.

\iim{c}$\orv \le \urv$.

\iim{d}The image set of $\urv$ is contained in that of $\orv$.
\end{theorem}

\begin{proof}\hskip.5em
Everything is an immediate consequence of Lemma~\ref{st:step}. 
\end{proof}

\vskip-2mm
\smallskip
\begin{theorem}\hskip.5em
\label{st:char}
The upper revised valuation $\urv$ is the lowest of the valuations $\val$ that lie above~$\orv$ and satisfy the equation $\val{}'=\val$.
\end{theorem}
\begin{proof}\hskip.5em
Theorem~\ref{st:rev} ensures that $\urv$ satisfies the inequality $\orv\le\urv$.
Besides, its being invariant by the transformation $\orvbis\mapsto\orvbis'$ means that 
it satisfies the equality $\urv{}'=\urv$.
It remains to see that $\orv \le \orvbis$ together with $\utvbis=\orvbis$ implies $\urv\le\orvbis$.
To this effect, it suffices to use Theorem~\ref{st:rev}.b to see that
the inequality $\orv \le \orvbis$ entails $\urv \le \urvbis$, and 
to combine this inequality with the equality $\urvbis = \orvbis$,
which follows from $\utvbis = \orvbis$ by the definition of $\urvbis$.
\end{proof}

\paragraph{3.2}
\textbf{Consistency.}\ensep
Let us look at the meaning of satisfying the equation $\val'=\val$.
Since the inequality $\val'\ge\val$ is always satisfied
---because of the \textit{tertium non datur} clauses, as it was seen in Lemma 3.1.c---
satisfying that equation is equivalent to satisfying the inequality $\val\ge\val'$,
\ie having
\begin{equation}
\label{eq:inequality}
\val_\lit \,\ge\, \min_{\latop{\lxt\in\clau}{\lxt\neq\lit}}\, \val_{\nt\lxt},\qquad\text{$\forall\lit\in\piset$ and $\forall\clau\in\doct$ with $\lit\in\clau$}.
\end{equation}
This is saying that the valuation~$\val$ is consistent with principle~{\small(\teof)} in~connection with all the implications of the form~(\ref{eq:pimplicant}) contained in the doctrine.
This motivates the following definition:\ensep
A valuation $\val$ is \dfc{consistent} with the doctrine $\doct$
\ifoi it satisfies the equation $\val'=\val$, 
where $\val'$ means the image of $\val$
by the transformation defined by (\ref{eq:vprime}).\ensep
As we have just remarked, this is equivalent to requiring that $\val$ satisfies the inequalities~(\ref{eq:inequality}).

\bigskip
The following results show that
this notion of consistency for a general valuation $\val$
agrees with the standard notion of consistency of a truth assignment
---namely that the Boolean formula (\ref{eq:cnf}) evaluates to true---
and it is stronger than the notion of definite consistency of a decision
that we defined in~\secpar{2.2}.

%
%
%
%
%

\smallskip
\begin{proposition}\hskip.5em
\label{st:truthass}
A truth assignment is consistent with the doctrine if and only if the corresponding all-or-none valuation $u$ satisfies the equality $u'=u$.
\end{proposition}
\begin{proof}\hskip.5em
Let $\Phi$ be the conjunctive normal form that expresses the doctrine.

Let us begin by seeing that the consistency of the truth assignment implies $u'_\lit=u_\lit$ for any $\lit\in\piset$.
When $u_\lit=1$, this is true because of parts~(c) and (d) of Lemma~\ref{st:step}.
When $u_\lit=0$, the truth of $\Phi$ requires that any clause~$\clau$ that contains $\lit$ must contain also some $\lxt\neq\lit$ with $u_\lxt=1$ and therefore $u_{\nt\lxt}=0$. By introducing this in (\ref{eq:vprime}) one gets $u'_\lit=0$.

Let us see now that the truth assignment not being consistent implies $u'\neq u$. The lack of consistency of the truth assignment means that there exists at least one clause $\clau$ that is not satisfied, \ie such that $u_\lit=0$ for any $\lit\in\clau$. By taking one such $\lit$ and noticing that $u_{\nt\lxt}=1$ for any $\lxt\in\clau\setminus\{\lit\}$ ---that exists because of {\small(D2)}--- one gets $u'_\lit=1>0=u_\lit$. 
\end{proof}

\smallskip
\begin{proposition}\hskip.5em
\label{st:decision}
Let $u$ be a partial truth assignment. If the equality $u' = u$ is satisfied, then
$u$ is definitely consistent with the doctrine.
\end{proposition}
\begin{proof}\hskip.5em
In accordance with the definition of definite consistency given in \secpar{2.2},
we have to check that for each $\clau\in\doct$ and every $\lit\in\clau$,
having $u_{\nt\lxt}=1$ for all $\lxt\in\clau\setminus\{\lit\}$
implies $u_\lit=1$. To this effect, it suffices to use the definition of~$u'$,
which gives $u'_\lit=1$, plus the assumed equality $u_\lit=u'_\lit$.
\end{proof}

\remark
The converse is not true: 
In fact, for a clause $\lit\lor\liit\lor\liiit$ the definition of definite consistency allows for having $u_\lit=0$ and $u_\liit=u_\liiit=\onehalf$, in which case one gets $u'_\lit=\onehalf\neq u_\lit$.
\ensep
\xtra
Having said that, one easily sees that $u$ being definitely consistent implies $\mu(u') = u$.

\bigskip
The following is one of the fundamental results of this article:

\smallskip
\begin{theorem}\hskip.5em
\label{st:dec}
A valuation $\val$ being consistent, \ie satisfying the equality $\val'=\val$,
entails the definite consistency of the associated decision of margin~$\mg$
---that is the partial truth assignment~$\mu_\mg(\val)$---
for any $\mg$ in the interval $0\le\mg\le1$.
\end{theorem}
\begin{proof}\hskip.5em
Recall that the consistency of $\val$ is equivalent to its satisfying 
the inequalities (\ref{eq:inequality}).
In accordance with the definition of definite consistency, we must show that
for each $\clau\in\doct$ and every $\lit\in\clau$, if all $\lxt\in\clau\setminus\{\lit\}$ are rejected for the decision of margin $\mg$ associated with $\val$, then $\lit$ is accepted for this decision.
Assume the contrary: $\lit$ is not accepted, that is
\begin{equation}
\label{eq:pnotaccepted}
\val_{\nt\lit} \,\ge\, \val_\lit-\mg.
\end{equation}
From (\ref{eq:inequality}) it follows that 
\begin{equation}
\label{eq:ineqp}
\val_\lit \,\ge\, \min_{\latop{\lxt\in\clau}{\lxt\neq\lit}} \val_{\nt\lxt} \,=\, \val_{\nt\liit},
\end{equation}
for some $\liit\in\clau\setminus\{\lit\}$.
Let us fix such a $\liit$.
By combining (\ref{eq:pnotaccepted}) and (\ref{eq:ineqp}) we get
\begin{equation}
\label{eq:pcomb}
\val_{\nt\lit} \,\ge\, \val_{\nt\liit} - \mg.
\end{equation}

Now, since $\liit\neq\lit$, the hypothesis that $\liit$ is rejected is saying that
\begin{equation}
\label{eq:qrejected}
\val_{\nt\liit} - \mg \,>\, \val_\liit.
\end{equation}
On the other hand, (\ref{eq:inequality}) entails that
\begin{equation}
\label{eq:ineqq}
\val_\liit \,\ge\, \min_{\latop{\lxt\in\clau}{\lxt\neq\liit}} \val_{\nt\lxt}
\,=\, \min\Big(\,\val_{\nt\lit}\,\,,\min_{\latop{\lxt\in\clau}{\lxt\not\in\{\lit,\liit\}}} \val_{\nt\lxt}\,\Big),
\end{equation}
from which (\ref{eq:pcomb}) and (\ref{eq:ineqp}) ---and the hypothesis that $\mg\ge 0$--- allow to conclude that
\begin{equation}
\label{eq:qcomb}
\val_\liit \,\ge\, \val_{\nt\liit} - \mg,
\end{equation}
in contradiction with (\ref{eq:qrejected}).
\end{proof}

\medskip
\begin{corollary}\hskip.5em
\label{st:dec_cor}
For any $\mg$ in the interval $0\cd\le\mg\cd\le1$, 
the decision of margin~$\mg$ associated with 
the upper revised valuation~$\urv$
---that is the partial truth assignment~$\mu_\mg(\urv)$---
is always definitely consistent.
\end{corollary}

\begin{proof}
It is simply a matter of recalling that~$\urv$
satisfies the hypothesis of the preceding theorem.
\end{proof}

\paragraph{3.3}
\textbf{Respect for consistent majority decisions and respect for unanimity.}\ensep
A very natural property to ask for is the following:
If the majority criterion applied to the original valuation~$\orv$
decides on each issue and is consistent with the doctrine,
then this decision should hold.
This property of \textbf{respect for consistent majority decisions}
will be obtained as a consequence of the following fact:

\medskip
\begin{proposition}\hskip.5em
\label{st:uv}
For any valuation $\orv$, and any consistent truth assignment $u$, one has
\begin{equation}
\label{eq:umax}
\max_{u_\lit=1}\, \urv_{\nt\lit} \,\le\, \max_{u_\liit=1}\, \orv_{\nt\liit}.
\end{equation}
\end{proposition}

\begin{proof}\hskip.5em
Since $\urv$ is obtained by iterating the transformation $\orv\mapsto\utv$,
it~suffices to show that this transformation has the following property analogous to (\ref{eq:umax}): \,$\utv_{\nt\lit} \,\le\, \max_{u_\liit=1} \orv_{\nt\liit}$\, whenever $u_\lit=1$. This follows immediately from (\ref{eq:vprime}) because the consistency of $u$ requires every clause $\clau$ that contains $\nt\lit$ with $u_\lit=1$ to contain also some $\liit$ with $u_\liit=1$, which ensures that $\min_{\lxt\in\clau, \lxt\neq\nt\lit}\,\orv_{\nt\lxt} \le \max_{u_\liit=1} \orv_{\nt\liit}$.
\end{proof}


\smallskip
\noindent
The property of respect for consistent majority decisions corresponds to the case $\theta\ge\onehalf\ge\theta-\mg$ of the following more general result:

\medskip
\begin{theorem}\hskip.5em
\label{st:majority}%
Assume that there exist $\theta\in(0,1)$ and $\mg\in[0,1)$ such that every $p\in\piset$ satisfies either $\orv_\lit>\theta\ge\theta-\mg>\orv_{\nt\lit}$ or, contrarily, $\orv_{\nt\lit}>\theta\ge\theta-\mg>\orv_\lit$. If the decision associated with $\orv$
(which contains no undecidedness and has a margin $\mg$) 
is consistent with the doctrine, then it agrees with the decision of margin $\mg$ associated with the upper revised valuation $\urv$.
\end{theorem}

\begin{proof}\hskip.5em
Let $u$ be the truth assignment that corresponds to the decision associated with $\orv$, that is:
\begin{equation}
u_\lit = \begin{cases}
1,&\text{if\, $\orv_\lit>\theta\ge\theta-\mg>\orv_{\nt\lit}$;}\\
0,&\text{if\, $\orv_{\nt\lit}>\theta\ge\theta-\mg>\orv_\lit$.}
\end{cases}
\end{equation}
From this definition it follows that
\begin{equation}
\label{eq:new_omaxlessthanmin}
\min_{u_\lit=1}\, \orv_\lit \,-\, \max_{u_\lit=1}\, \orv_{\nt\lit} \,>\, \mg.
\end{equation}
By combining this inequality with (\ref{eq:umax}), which holds because of the consistency of $u$, and part~(c) of Theorem~\ref{st:rev}, it follows that
\begin{equation}
\min_{u_\lit=1}\, \urv_\lit \,-\, \max_{u_\lit=1}\, \urv_{\nt\lit}\,>\, \mg,
\end{equation}
which implies $\urv_\lit - \urv_{\nt\lit} > \mg$ whenever $u_\lit=1$.
\end{proof}

\remark
When $\orv$ is balanced,
the condition $\orv_\lit > \onehalf > \orv_{\nt\lit}$ is equivalent to requiring simply that $\orv_\lit > \orv_{\nt\lit}$.
Generally speaking, however, the latter condition cannot be substituted for the former in the preceding result.

\smallskip
\bigskip
According to part~(c) of Theorem~\ref{st:rev}, $\orv_\lit=1$ implies $\urv_\lit=1$. The following result shows that the converse implication holds when $\orv$ is an aggregate of 
definitely consistent partial truth assignments.
By an aggregate we mean a convex combination of several components $\orvk$, as in \textup{(\ref{eq:convcomb})}.
\smallskip
\begin{proposition}\hskip.5em
\label{st:unampro}
Assume that $\orv$ is an aggregate of
definitely consistent partial truth assignments.
In this case, having $\urv_{\lit} = 1$ implies $\orv_{\lit} = 1$. 
\end{proposition}
\begin{proof}\hskip.5em
We will proceed by induction on the iterates~$\ntv{n}$. More specifically, we aim at showing that 
\begin{equation}
\label{eq:unamn}
\text{for any $\lit\in\piset$, \,the equality $\ntv{n}_{\lit} = 1$ \,implies\, $\orv_{\lit} = 1,$}
\stepcounter{equation}
\tag{\theequation$\smash{{}_n}$}
\end{equation}
\newcommand\equnamn[1]{\theequation$\smash{}_{#1}$}%
\newcommand\sqz{\kern-.175pt\ }%
\newcommand\sqzz{\kern-.875pt\ }%
where\sqz we\sqz emphasize\sqz that\sqz the\sqz statement\sqz includes\sqz the\sqz quantifier\sqzz ``for\sqzz any\sqzz $\lit\cd\in\piset$''\!.\linebreak
For $n=0$ this is trivially true since $\smash{\ntv{0}=\orv}$.
For $n\ge1$, (\ref{eq:unamn}) can be obtained from (\equnamn{n-1}) in the following way:
Let us assume that $\smash{\ntv{n}_{\lit} = 1}$.
If~$\smash{\ntv{n-1}_{\lit} = 1}$, the conclusion follows directly from (\equnamn{n-1}).
Otherwise, the~inequality $\smash{1 = \ntv{n}_{\lit} > \ntv{n-1}_{\lit}}$ entails the existence of a proper clause~$\clau$, \ie a clause different from the ones associated with the \textit{tertium non datur} principle, such that $\ntv{n-1}_{\nt\lxt}=1$ for any $\lxt\in\clau\setminus\{\lit\}$. Now, (\equnamn{n-1}) ensures that $\orv_{\nt\lxt}=1$.
\linebreak
Notice that this implies $\utv_\lit=1$, which entails $\orv_\lit=1$ once (\equnamn{1}) is established; however, establishing (\equnamn{n}) for any $n\ge1$ requires making use of the fact that $v$ has the form~(\ref{eq:convcomb}), where the $\orvk$ are definitely consistent partial truth assignments. Since $\pesk\ge0$ and $\sum_k\pesk=1$, it is clear that the equality $\orv_{\nt\lxt}=1$ entails $\orvk_{\nt\lxt}=1$ whenever $\pesk>0$,
\ie for any $k$ that matters.
Now, since
partial truth assignments are balanced,
it~follows that $\orvk_\lxt=0$. On the other hand, since the $\orvk$ are
definitely consistent with the doctrine,
and the preceding conclusion is valid for any $\lxt\in\clau\setminus\{\lit\}$, it follows that $\orvk_\lit=1$. Finally, this certainly implies $\orv_\lit=1$.
\end{proof}

\smallskip
\noindent
An important consequence of the preceding result is the following property of \textbf{respect for unanimity}:

\smallskip
\begin{theorem}\hskip.5em
\label{st:unanimity}
Assume that $\orv$ is an aggregate of
definitely consistent partial truth assignments.
In this case, having $\orv_\lit = 1$ implies that $\lit$ is accepted by the basic decision associated with the upper revised valuation~$\urv$.
\end{theorem}
\begin{proof}\hskip.5em
We must show that $\urv_\lit>\urv_{\nt\lit}$. In the present case we know that $\urv_\lit = 1$, so that the proof reduces to showing that $\urv_{\nt\lit}<1$. This inequality is easily obtained by contradiction. In fact, according to Proposition~\ref{st:unampro}, $\urv_{\nt\lit} = 1$ would imply $\orv_{\nt\lit} = 1$, that is incompatible with $\orv_\lit = 1$ because $\orv_\lit+\orv_{\nt\lit}=1$, which derives from the analogous equality satisfied by the components $\orvk$. 
\end{proof}


\paragraph{3.4}
\textbf{Monotonicity.}\ensep
In this subsection we look at the effect of raising the original value of a particular literal $\lit$ 
\emph{without any change in the others.}
We~aim at showing that in these circumstances the acceptability of $\lit$ either increases or stays constant.

\begin{lemma}\hskip.5em
\label{st:monoLem}
Consider the dependence of $\urv_\liit$ on $\orv_\lit$ when $\orv_\lxt$ is kept constant for all $\lxt\neq\lit$.
For any $\liit\in\piset$, this dependence has the following form:
there exist $a$ and $b$ with $0\le a\le b\le 1$ such that
\begin{equation}
\label{eq:monoLem}
\urv_\liit \,=\,
\begin{cases}
a, &\text{for $0\le\orv_\lit\le a$,}\\
\orv_\lit, &\text{for $a\le\orv_\lit\le b,$}\\
b, &\text{for $b\le\orv_\lit\le 1$.}
\end{cases}
\end{equation}
\end{lemma}

\begin{proof}\hskip.5em
According to part~(d) Theorem~\ref{st:rev}, the graph of $\urv_\liit$ as a function of~$\orv_\lit$ is contained in the union of the horizontal lines $\urv_\liit=\orv_\lxt$ ($\lxt\in\piset\setminus\{\lit\}$) together with the diagonal one $\urv_\liit=\orv_\lit$. On the other hand, part~(a) of that theorem ensures that the function $\orv_\lit\mapsto\urv_\liit$ is continuous. These constraints leave no other possibility than the pattern~(\ref{eq:monoLem}).
\end{proof}


In order to analyse the effect of raising the value of $\orv_\lit$ we will use the following notation and terminology:
$\orvz$ denotes a modified valuation that differs from $\orv$ only in that $\orvz_\lit > \orv_\lit$.
The objects associated with $\orvz$ will be referred to by means of a tilde.
For any magnitude $\magn$ that depends on~$\orv_\lit$,
the statement 
``$\magn$~\dfc{stays constant}'' means that 
there exists $\ptit>0$ such that 
for~any~$\orvz_\lit$ in the interval
$\orv_\lit\le\orvz_\lit\le\orv_\lit+\ptit$ 
one has $\magnz=\magn$;\ensep
similarly, the~statement 
``$\magn$~\dfc{increases}'' means that 
there exists $\ptit>0$ such that 
for~any~$\orvz_\lit$ in the interval
$\orv_\lit<\orvz_\lit\le\orv_\lit+\ptit$ 
one has $\magnz>\magn$.\ensep
Notice that
these definitions 
consider only values of $\orvz_\lit$ at the right of~$\orv_\lit$.

\smallskip
\begin{lemma}\hskip.5em
\label{st:monoLemm}
Let us allow $\orv_\lit$ to grow while $\orv_\lxt$ is kept constant for $\lxt\neq\lit$.
If there exists $\liit\in\piset$ such that $\urv_\liit$ increases, then $\urv_\lit=\orv_\lit$.
\end{lemma}

\begin{proof}\hskip.5em
We will proceed by contradiction. Let us assume that $\urv_\lit\neq\orv_\lit$. According to part~(c) of Theorem~\ref{st:rev}, it must be $\urv_\lit>\orv_\lit$. This allows to choose $\orvz_\lit$ so that $\orv_\lit<\orvz_\lit<\urv_\lit$. The hypothesis that $\urv_\liit$ increases ensures that $\urv_\liit<\urvz_\liit$ (this is true not only
for~$\orvz_\lit$ in a neighbourhood at the right of $\orv_\lit$,
but also for any $\orvz_\lit>\orv_\lit$,
since part~(b) of Theorem~\ref{st:rev} ensures that $\urvz_\liit$ is a non-decreasing function of $\orv_\lit$).
Let us consider now the valuation $\orvbis$ that coincides with $\orv$ for all literals except possibly $\lit$, for which we set $\orvbis_\lit=\urv_\lit$. Since $\orvz_\lit<\urv_\lit$, part~(b) of Theorem~\ref{st:rev} ensures that $\urvz_\liit\le\urvbis_\liit$. We will arrive at contradiction by showing that $\urvbis_\liit=\urv_\liit$. In fact, by combining this equality with some of the preceding inequalities, namely $\urv_\liit<\urvz_\liit\le\urvbis_\liit$, one would arrive at the false conclusion that $\urv_\liit<\urv_\liit$.

The claim that $\urvbis_\liit=\urv_\liit$ is again a consequence of part~(b) of Theorem~\ref{st:rev}: Since $\orv\le\orvbis\le\urv$, that result ensures that $\urv\le\urvbis\le\urv{}^*$. On the other hand, the invariance of $\urv$ by the transformation $u\mapsto u'$ entails $\urv{}^*=\urv$.
\end{proof}

\smallskip
\begin{theorem} 
\label{st:monoThm}
Let us allow $\orv_\lit$ to grow while $\orv_\lxt$ is kept constant for $\lxt\neq\lit$.
In these circumstances, the acceptability of $\lit$, \,
that is the difference $\urv_\lit-\urv_{\nt\lit}$, \,
either increases or stays constant.
\end{theorem}

\begin{proof}\hskip.5em
By part~(b) of Theorem~\ref{st:rev}, $\urv_\lit$ and $\urv_{\nt\lit}$ have the property that each of them either increases or stays constant. When $\urv_{\nt\lit}$ stays constant we are done. Now, when $\urv_{\nt\lit}$ increases, Lemma~\ref{st:monoLemm} ensures that $\urv_\lit=\orv_\lit$. Since\linebreak 
$\urvz_\lit\ge\orvz_\lit>\orv_\lit$, it follows that $\urv_\lit$ increases too. However, according to Lemma~\ref{st:monoLem} the only possible way for both $\urv_\lit$ and $\urv_{\nt\lit}$ to increase is by having $\urvz_\lit=\urvz_{\nt\lit}=\orvz_\lit$ for any $\orvz_\lit$ in an interval of the form $\orv_\lit\le\orvz_\lit\le\orv_\lit+\ptit$ with $\ptit>0$, which entails that $\urv_\lit-\urv_{\nt\lit}$ stays constant (and equal to zero).
\end{proof}

\smallskip
\begin{corollary}\hskip.5em
\label{st:monoCorr}
Assume that the valuation \,$\orv$\, is modified into a new one \,$\orvz$\, such that
\begin{equation}
\label{eq:mona}
\orvz_\lit \,>\, \orv_\lit,\qquad \orvz_\liit \,=\, \orv_\liit,\quad \forall q\in\piset\setminus\{\lit\},
\end{equation}
and let $\mg$ be any number in the interval $0\le\mg\le1$. If $\lit$ is accepted \textup[resp.~not rejected\,\textup] in the decision of margin $\mg$ associated with~$\urv$, then it is also accepted \textup[resp.~not rejected\,\textup] in~the decision of margin $\mg$ associated with~$\urvz$.
\end{corollary}

\section{Technical issues and supplementary topics}

\paragraph{4.1}
\textbf{Which conjunctive normal form?}\ensep
In this subsection we look at the effect of the conjunctive normal form $\Phi$ that expresses the doctrine. In~this connection, we will be led to make a particular choice, namely the so-called Blake canonical form.

In~the following $\Phiz$ stands for an alternative conjunctive normal form (with literals in the same set $\piset$) and we systematically use a tilde to denote the objects associated with $\Phiz$. In particular, $\urv$ and $\urvz$ mean the upper revised valuations obtained from a given $\orv$ by using respectively $\Phi$ and $\Phiz$. For the moment, $\Phi$ and $\Phiz$ need not be logically equivalent to each other.

\bigskip
If one has $\urv=\urvz$ for any $\orv$ we will say that $\Phi$ and $\Phiz$ are \dfc{$\ast$-equivalent.}\ensep

\smallskip
\begin{proposition}\hskip.5em
\label{st:equivalence}
$\ast$-Equivalence implies logical equivalence.
\end{proposition}
\begin{proof}\hskip.5em
$\Phi$ and $\Phiz$ not being logically equivalent means that there exists a truth assignment $u$ that makes one of them, say $\Phi$, true, and the other, $\Phiz$, false. According to Proposition~\ref{st:truthass}, this entails that $u^*=u$ but $\uurvz\neq u$, so~that $\Phi$ and $\Phiz$ are not $\ast$-equivalent.
\end{proof}

\noindent
However, \emph{logical equivalence does not imply $\ast$-equivalence}.
As a simple example, we can take $\Phi=(\lit\lor\liit)\land(\lit\lor\liit\lor\liiit)$ and $\Phiz = \lit\lor\liit$, for which one easily checks that
$\urvz_\liiit = \orv_\liiit < \urv_\liiit$ as soon as $\orv_\liiit < \min(\orv_{\nt\lit},\orv_{\nt\liit})$.
\ensep
The following result can be seen as a generalization of what happens in this example.

\smallskip
\begin{proposition}\hskip.5em
\label{st:opent}
If all clauses of $\Phiz$ are present also in $\Phi$, \ie $\doctz\sbseteq\doct$, then $\urvz\le\urv$.
\end{proposition}
\begin{proof}\hskip.5em
We will proceed by induction along the iteration that eventually gives the upper revised valuation $\urv$. As in \secpar{3.1}, we denote the iterates by~$\ntv{n}$. We aim at showing that $\ntvz{n}\le\ntv{n}$ for any $n\ge0$. For $n=0$ this is satisfied as an equality, since we are starting from the same valuation $\orv$. In~order to go from $n-1$ to $n$, it suffices to compare $\ntvz{n}$ and $\ntv{n}$ through the valuation $\val$ defined by
\begin{equation}
\val_\lit \,=\, \max_{\latop{\clau\in\doct}{\clau\ni\lit}}\, \min_{\latop{\lxt\in\clau}{\lxt\neq\lit}} \,\ntvz{n-1}_{\nt\lxt},
\end{equation}
(with a tilde in $\ntvz{n-1}$, but not in $\doct$).
In fact, on the one hand, by comparing the preceding expression with the one that defines $\ntvz{n}$ in terms of $\ntvz{n-1}$, one easily sees that $\ntvz{n}\le\val$, since
more clauses entail a higher maximum. 
On~the other hand, Lemma~\ref{st:step}.b allows to derive the inequality $\val\le\ntv{n}$ from the hypothesis that $\ntvz{n-1}\le\ntv{n-1}$.
\end{proof}

\medskip
Let us look at the meaning of the hypothesis contained in principle~{\small(\teof)},
namely the satisfiability (within the doctrine under consideration)
of the conjunction $\bigwedge_{\lxt\in\spiset}\lxt$ that implies $\lit$.
Not requiring this satisfiability would lead to such non-senses as giving a certain degree of belief to any $\lit$ just because of the tautological but `improper' implication $\lit\leftarrow(\liit\land\nt\liit)$ (for an arbitrary $\liit$).
\ensep
Now, the conjunction that appears in~(\ref{eq:pimplicant}) is $\bigwedge_{\lxt\in\clau,\,\lxt\neq\lit} \nt\lxt$, whose
not being satisfiable is equivalent to $\bigvee_{\lxt\in\clau,\,\lxt\neq\lit} \lxt$ being always satisfied, in~which case the clause $\bigvee_{\lxt\in\clau} \lxt$ could have been replaced by the stronger one $\bigvee_{\lxt\in\clau,\,\lxt\neq\lit} \lxt$. 
So, the hypothesis that the right-hand side of~(\ref{eq:pimplicant}) is satisfiable whenever $\lit\in\clau\in\doct$ amounts to say that every clause has the property that  the corresponding set $\clau\in\doct$ has no proper subset $\clau'$ with $\phi(\clau')=\bigvee_{\lxt\in\clau'} \lxt$ already entailed by the doctrine. 
In other words, every clause should be a \dfc{prime implicate} of the doctrine under consideration.
In~\cite{dl10b, lipo, dl} one uses the equivalent terminology of saying that $\nt\clau = \{\,\nt\lxt\mid\lxt\in\clau\,\}$ is a `minimal inconsistent set', whereas in \cite{np08, np10} $\nt\clau$ is said to be a `critical family'.
From now on we add the following assumption to those adopted in \secpar{2.1}:

\smallskip
\iim{D4} The conjunctive normal form that expresses the doctrine is made of prime implicates.
\smallskip

\noindent
Such conjunctive normal forms will be referred to as \dfd{prime conjunctive normal forms}.
As we have seen at the beginning of this paragraph,
they are the only ones that propagate belief in a proper way, without making it up.

\bigskip
Among all the prime conjunctive normal forms equivalent to a given doctrine,
we will take as a natural choice
the \dfd{Blake canonical form}, so called after Archie Blake, who introduced it in 1937~\cite{blake}.
Given a formula $f$, the Blake canonical form of~$f$ is the conjunctive normal form which is made of all the prime implicates of $f$. 
The following statement gives its main property
in connection with the upper revised valuation~$\urv$:

\smallskip
\begin{proposition}\hskip.5em
\label{st:bcfmax}
Among all the prime conjunctive normal forms logically equivalent to a given doctrine,
the Blake canonical form has the property of
giving the greatest possible upper revised valuation.
\end{proposition}
\begin{proof}\hskip.5em
It follows immediately from the definitions by virtue of Proposition\,\ref{st:opent}.
\end{proof}




\bigskip
The next theorem gives a systematic method for obtaining the Blake canonical form of a formula~$f$. The procedure will start from any conjunctive normal form $\Phiz$ logically equivalent to $f$ ---which is quite easy to obtain--- and it will apply  repeatedly the two following operations, where $\phi,\psi$ stand for generic disjunctions of literals:

\newcommand\ditm[2]{\hangindent2\parindent\hangafter1\hskip0pt\hbox to6.3em{\textbf{#1}:\hss}\strut#2\strut}

\smallskip
\ditm{Absorption}{replacing\, $\phi\land(\phi\lor\psi)$ \,by\, $\phi$.}

\halfsmallskip
\ditm{Resolution}{replacing\, $(\phi\lor\lit)\land(\psi\lor\nt\lit)$ \,by\, $(\phi\lor\lit)\land(\psi\lor\nt\lit)\land(\phi\lor\psi)$ under the proviso that $\phi$ and $\psi$ do not contain respectively a literal and its negation, 
and that \,$\phi\lor\psi$\, is not absorbed by a clause already present.}

\smallskip
\noindent
In the dual context of disjunctive normal forms, the operations analogous to resolution and absorption are known respectively as consensus and subsumption.

\smallskip
\begin{theorem}[Blake, 1937; Samson, Mills, 1954; Quine, 1955--59; see \cite{brown,marquis}]\hskip.25em 
\label{st:resabs}%
\leavevmode\hbox{In~order} to obtain the Blake canonical form of a formula $f$ it suffices to take any conjunctive normal form $\Phiz$ logically equivalent to $f$ and to transform it by applying repeatedly the operations of absorption and resolution until no further application is possible (which happens after a finite number of steps).
\end{theorem}


\remarks
1. In order to check a given conjunctive normal form~$\Phiz$ for its being prime, it is a matter of deriving the Blake canonical form $\Phi$ by following the procedure of absorption and resolution and checking whether $\Phi$ contains all the clauses of $\Phiz$; in other words, checking that none of the original clauses disappears by absorption in the process.

2. According to the special convention {\small(D3)} adopted at the end of \secpar{2.1}, we~systematically supplement the Blake canonical form with all the \textit{tertium non datur} clauses (which are neutral elements for the operation of resolution, and are prime under the assumption {\small(D2)} of absence of unit clauses).
Under this convention, in the definition of the operation of resolution one can drop the proviso that $\phi$ and $\psi$ must not contain respectively a literal and its negation; in fact, in this case $\phi\lor\psi$ is absorbed by one the \textit{tertium non datur} clauses.
The context will always make clear whether we mean the standard Blake canonical one or the supplemented one.

\xtra
\paragraph{4.2}
\textbf{Disjoint resolution.}\ensep
In some cases the Blake canonical form can be quite long.
In fact, it is well known that generally speaking 
its length can grow exponentially in the number of variables.
This raises the question whether there are shorter prime conjunctive normal forms 
$\ast$-equiv\-alent to~it. 
Generally speaking, it need not be so.
A~simple example is the conjunctive normal form
$\Phiz = (\nt\lit\lor\nt\liit\lor\liiit)\land 
(\nt\lit\lor\nt\liiit\lor\livt)$,
for~which one can easily check that the Blake canonical form is
$\Phi = (\nt\lit\lor\nt\liit\lor\liiit)\land 
(\nt\lit\lor\nt\liiit\lor\livt)\land(\nt\lit\lor\nt\liit\lor\livt)$,
and that 
$\urv_{\nt\lit} > \urvz_{\nt\lit}$ as soon as $\min(\orv_\liit,\orv_{\nt\livt}) > \max(\orv_\liiit, \orv_{\nt\liiit})$.
\ensep
Having said that, the next results identify certain situations where shorter conjunctive normal forms can be shown to be \hbox{$\ast$-equivalent} to the Blake canonical form.
In this connection, we will use the following terminology:
A \dfc{disjoint resolution} means a resolution operation as above 
but with $\phi$ and $\psi$ having no literals in common. A conjunctive normal form will be called \dfc{disjoint-resolvable} when a suitable choice of which clauses to resolve at every stage allows to arrive at the Blake canonical form by only making use of disjoint resolution and absorption.

\smallskip
\begin{proposition}\hskip.5em
\label{st:resolution}
Let $\Phi$ and $\Phiz$ be 
conjunctive normal forms.
If $\Phi$ is obtained from $\Phiz$ by a single disjoint resolution operation,
then $\urv = \urvz$.
\end{proposition}
\begin{proof}\hskip.5em
Clearly, $\doctz\sbseteq\doct$. 
So Proposition~\ref{st:opent} ensures that $\urvz\le\urv$.
In order to prove the equality we will make use of the characterization of $\urv$ given by Theorem~\ref{st:char}: $\urv$ is the lowest valuation $\val$ that satisfies $\valt=\val\ge\orv$.
So, we will be done if we show that the latter conditions are satisfied when we put $\val=\urvz$. From now on we fix $\val$ to mean $\urvz$.
What we know is that $\valtz=\val\ge\orv$, where in principle $\valtz$ differs from $\valt$ in that they use different conjunctive normal forms.
So the problem reduces to showing that $\valt=\valtz$.

Now, the two conjunctive normal forms in consideration differ only in that $\doct = \doctz \cup \{\clauiii\}$, where $\clauiii = (\clauzi \setminus \{\liit\}) \cup (\clauzii \setminus \{\nt\liit\})$ with $\clauzi,\clauzii\in\doctz$ containing respectively $\liit$ and $\nt\liit$.
Clearly, $\valt_\lit$ cannot differ from $\valtz_\lit$ except for $\lit\in\clauiii$, in which case we have
\begin{equation}
\label{eq:dri}
\valt_\lit \,\,=\,\, \max\bigg( \valtz_\lit\,, \min_{\latop{\lxt\in\clauiii}{\lxt\neq\lit}} \,\val_{\nt\lxt} \bigg).
\end{equation}
So we should show that
\begin{equation}
\label{eq:drii}
\min_{\latop{\lxt\in\clauiii}{\lxt\neq\lit}} \,\val_{\nt\lxt} \,\,\le\,\, \valtz_\lit,\qquad\forall\lit\in\clauiii.
\end{equation}
The hypothesis of disjoint resolution means that either $\lit\in\clauzi\setminus\{\liit\}$ or $\lit\in\clauzii\setminus\{\nt\liit\}$, but not both.
The two cases are analogous to each other, so~we will consider only the second one.
In order to obtain (\ref{eq:drii}) we begin by noticing that
\begin{equation}
\label{eq:driii}
\min_{\latop{\lxt\in\clauiii}{\lxt\neq\lit}} \,\val_{\nt\lxt}
\,\,=\,\,
\min\bigg(
\min_{\latop{\latop{\lxt\in\clauzzi}{\lxt\neq\liit}}{\lxt\neq\lit}} \,\val_{\nt\lxt}
\,\,,\,
\min_{\latop{\latop{\lxt\in\clauzzii}{\lxt\neq\nt\liit}}{\lxt\neq\lit}} \,\val_{\nt\lxt}
\bigg).
\end{equation}
Now, since we are assuming $\lit\notin\clauzi$, we have
\begin{equation}
\label{eq:driv}
\min_{\latop{\latop{\lxt\in\clauzzi}{\lxt\neq\liit}}{\lxt\neq\lit}} \,\val_{\nt\lxt}
\,\,=\,\,
\min_{\latop{\lxt\in\clauzzi}{\lxt\neq\liit}} \,\val_{\nt\lxt}
\,\,\le\,\,
\valtz_\liit
\,\,=\,\,
\val_\liit,
\end{equation}
where we used that $\valtz = \val$. Finally, by plugging (\ref{eq:driv}) into (\ref{eq:driii}) we get
\begin{equation}
\label{eq:drv}
\min_{\latop{\lxt\in\clauiii}{\lxt\neq\lit}} \,\val_{\nt\lxt}
\,\,\le\,\,
\min\bigg( \val_\liit\,,
\min_{\latop{\latop{\lxt\in\clauzzii}{\lxt\neq\nt\liit}}{\lxt\neq\lit}} \,\val_{\nt\lxt}
\bigg)
\,\,=\,\,
\min_{\latop{\lxt\in\clauzzii}{\lxt\neq\lit}} \,\val_{\nt\lxt}
\,\,\le\,\,
\valtz_\lit,
\end{equation}
as desired.
\end{proof}

\halfsmallskip
\begin{corollary}\hskip.5em
Any disjoint-resolvable prime conjunctive normal form is $\ast$-equi\-valent to the Blake canonical form.
\end{corollary}
\begin{proof}\hskip.5em
Let $\Phi$ and $\Phiz$ denote respectively the Blake canonical form and a disjoint-resolvable prime conjunctive normal form. Since we can go from $\Phiz$ to $\Phi$ via disjoint resolution and absorption, Propositions~\ref{st:resolution} and~\ref{st:opent} ensure that $\urv \le \urvz$.\ensep
On the other hand, the reverse inequality is guaranteed by Proposition~\ref{st:bcfmax}, 
since the disjoint resolvability of $\Phiz$ certainly ensures that $\Phiz$ is logically equivalent to $\Phi$.
\end{proof}

\smallskip
\xtra
Related to the preceding results, one can see that disjoint resolvability ensures not only $*$-equivalence, but also equivalence with respect to the property of definite consistency of a partial truth assignment:
\smallskip
\begin{proposition}\hskip.5em
\label{st:intrinsic}
A partial truth assignment is definitely consistent
for a disjoint-resolvable conjunctive normal form
\,if and only if\,
it~is definitely consistent for the corresponding Blake canonical form.
\end{proposition}
\begin{proof}\hskip.5em
Let $\Phi$ and $\Phiz$ denote respectively the Blake canonical form and a disjoint-resolvable conjunctive normal form. 
\ensep
The claim that definite consistency with $\Phiz$ implies definite consistency with $\Phi$ follows from the easily checked fact that definite consistency 
is preserved by the operations of absorption and disjoint resolution.
\ensep
The converse implication is a simple consequence of the 
definition of the Blake canonical form, which ensures that 
every clause of $\Phiz$ contains, as a set of literals, some clause of $\Phi$.
\end{proof}

\smallskip
\remark The definition of disjoint resolvability implies going through (a special form of) the process of generating the Blake canonical form. On the other hand,
the latter is also required for checking whether a given conjunctive normal form is prime.
It would certainly be very interesting to be able to recognise these properties without having to go through the process of generating
the Blake canonical form.


\paragraph{4.3}
\textbf{Getting there in a single step.}\ensep
Since the Blake canonical form of a doctrine contains all of its prime implicates,
one might think that this conjunctive normal form will always lead to $\urv$ in a single step,
that is $\urv = \utv$.\linebreak
Generally speaking, however, it need not be so.

\smallskip
A simple example is the doctrine
\begin{equation}
\label{eq:doubleopposition}
(p\lor q\lor r) \,\land\, (\nt q\lor\nt r).
\end{equation}
This conjunctive normal form is already the Blake canonical form, since the two clauses in it are related to each other by a double opposition, which precludes the operation of resolution. The one-step revision transformation is as follows:
\begin{alignat}{4}
\label{eq:one-step-do-p}
&\utv_p \,&&=\, \max\,\big(\orv_p, \min(\orv_{\nt q},\orv_{\nt r})\big),\qquad
&&\utv_{\nt p} \,&&=\, \orv_{\nt p},
\\
\label{eq:one-step-do-q}
&\utv_q \,&&=\, \max\,\big(\orv_q, \min(\orv_{\nt p}, \orv_{\nt r})\big),\qquad
&&\utv_{\nt q} \,&&=\, \max\,(\orv_{\nt q}, \orv_r),
\\
\label{eq:one-step-do-r}
&\utv_r \,&&=\, \max\,\big(\orv_r, \min(\orv_{\nt p}, \orv_{\nt q})\big),\qquad
&&\utv_{\nt r} \,&&=\, \max\,(\orv_{\nt r}, \orv_q),
\end{alignat}
Using these equations, one easily checks, for instance, that the valuation
$(\orv_p, \orv_q, \orv_r; \orv_{\nt p}, \orv_{\nt q}, \orv_{\nt r}) = (.0, .4, .0; .0, .6, .3)$
results in
$(\utv_p, \utv_q, \utv_r; \utv_{\nt p}, \utv_{\nt q}, \utv_{\nt r}) = (.3, .4, .0; .0, .6, .4)$
and 
$(\utv'_p, \utv'_q, \utv'_r; \utv'_{\nt p}, \utv'_{\nt q}, \utv'_{\nt r}) = (.4, .4, .0; .0, .6, .4)$,
where $\utv'_p$ is strictly larger than $\utv_p$.

\smallskip
A doctrine whose Blake canonical form does not satisfy $\urv=\utv$ may raise suspicion.
This is indeed the case of the doctrine (\ref{eq:doubleopposition}). Let us see why.
By composing the transformation (\ref{eq:one-step-do-p}--\ref{eq:one-step-do-r}) with itself one gets
\begin{equation}
\label{eq:two-step-do}
\utv'_p = \max\,\big(\orv_p, \min(\orv_q,\orv_r), \min(\orv_q,\orv_{\nt q}), \min(\orv_{\nt r},\orv_r), \min(\orv_{\nt r},\orv_{\nt q})\big).
\end{equation}
Now, for the numerical values given after (\ref{eq:one-step-do-p}--\ref{eq:one-step-do-r}), the maximum of the right-hand side of (\ref{eq:two-step-do}) is achieved by the term $\min(\orv_q,\orv_{\nt q})$. This is quite embarrasing, since it corresponds to deriving belief about $p$ from the unsatisfiable conjunction $q\land\nt q$, which goes against the proviso stipulated in principle~{\small(P)}. One can argue that (\ref{eq:two-step-do}) is the result of two steps that separately comply
with that condition. However, one cannot but remain somehow unconvinced.

\medskip
A sufficient condition for having $\urv_p=\utv_p$ is provided by the following result.
From now on, we use the notation $\clau\res{\liit}\clau'$ to represent a set of the form $(\clau\setminus\{\liit\})\cup(\clau'\setminus\{\nt \liit\})$.

\begin{theorem}\hskip.35em 
\label{st:floretamillorada}
Let us assume that the following condition is satisfied for a given $\lit\in\piset$:\ensep

\halfsmallskip
\vskip.3ex
{\parskip0pt
For~any $\clau,\clau'\in\doct$ and $\liit\in\piset\setminus\{\lit\}$ satisfying $\lit,\liit\in\clau$ and $\nt\liit\in\clau'$,
\par
one has either \textup{(a)} or \textup{(b)} or both of them:
\par
\vskip.3ex
\ \textup{(a)}~there exists $\clau_1\in\doct$ such that $\lit\in\clau_1\sbseteq\clau\res{\liit}\clau'$;\par
\ \textup{(b)}~$\urv_\lit > \min_{\strut\rule{0pt}{1.5ex}\lxt\,\in\,\clau\ress{\liit}\clau'\,\setminus\,\{p\}}\urv_{\nt\lxt}$.
}

\noindent
In~such a situation one is ensured to have $\urv_\lit=\utv_\lit$.
\end{theorem}

\begin{proof}\hskip.5em
We will argue by contradiction. More specifically, we will assume 
$\urv_\lit=\ntv{n}_\lit>\ntv{n-1}_\lit$ with $n\ge2$, and we will see that it leads to contradiction.
The formula that defines $\ntv{n}_\lit$ allows us to write
\begin{equation}
\label{eq:vbiprime}
\urv_\lit \,=\,
\max_{\substack{\clau\in\doct\\\clau\ni\lit}}\,
\min_{\substack{\lxt\in\clau\\\lxt\neq\lit}} \,\ntv{n-1}_{\nt\lxt}
\,=\, \min_{\substack{\lxt\in\clau_0\\\lxt\neq\lit}} \,\ntv{n-1}_{\nt\lxt},
\end{equation}
where $\clau_0$ stands for any clause that realizes the maximum of the middle expression.
In the sequel we will refer to such a clause as a ‘realizer’.
For every realizer $\clau_0$ it will be useful to consider the following partition of it:
$\clau_0=\{\lit\}\cup\dot\clau_0\cup\ddot\clau_0$, where
$\dot\clau_0$ and $\ddot\clau_0$ are formed by those literals $\lxt\in\clau_0\setminus\{\lit\}$ that satisfy respectively $\ntv{n-2}_{\nt\lxt} < \urv_\lit$ and $\ntv{n-2}_{\nt\lxt} \ge \urv_\lit$.
By taking into account that $\ntv{n-1}\ge\ntv{n-2}$, one easily checks that
\begin{equation}
\label{eq:vbiprimedots}
\urv_\lit
\,=\, 
\min\Big(
\min_{\lxt\in\dot\clau_0} \,\ntv{n-1}_{\nt\lxt}\,,\,
\min_{\lxt\in\ddot\clau_0} \,\ntv{n-2}_{\nt\lxt}
\,\Big).
\end{equation}

\smallskip
For $\dot\clau_0 = \emptyset$, the right-hand side of (\ref{eq:vbiprimedots}) reduces to the expression 
whose maximization defines $\ntv{n-1}_p$, which gives $\urv_\lit \le \ntv{n-1}_\lit$, 
in~contradiction to our assumption.
\ensep
Therefore, our aim will be reached if we show that for any realizer $\clau_0$ with $\dot\clau_0\neq\emptyset$,
either the equality (\ref{eq:vbiprimedots}) leads itself to contradiction,
or there exists another realizer $\clau_1$ with the property that $\dot\clau_1$ is strictly contained in $\dot\clau_0$, 
which would lead to the empty set in a finite number of steps.

So, let $\clau_0$ be a realizer with $\dot\clau_0\neq\emptyset$,
and let $\liit$ be any element of $\dot\clau_0$.
Similarly to (\ref{eq:vbiprime}), we can write
\begin{equation}
\label{eq:vprime_noalpha}
\ntv{n-1}_{\nt\liit} \,=\,
\max_{\substack{\clau\in\doct\\\clau\ni\nt\liit}}\,
\min_{\substack{\lxt\in\clau\\\lxt\neq\nt\liit}} \,\ntv{n-2}_{\nt\lxt}
\,=\, \min_{\substack{\lxt\in\clau'_0\\\lxt\neq\nt\liit}} \,\ntv{n-2}_{\nt\lxt},
\end{equation}
where $\clau'_0$ stands for any clause that realizes the maximum of the middle expression.
By introducing (\ref{eq:vprime_noalpha}) in (\ref{eq:vbiprimedots}),
we get
\begin{equation}
\label{eq:vbiprimebis}
\ntv{n}_\lit \,=\,
\min\Big(
\min_{\lxt\in\dot\clau_0\setminus\{\liit\}} \,\ntv{n-1}_{\nt\lxt}\,,\,
\min_{\lxt\in\ddot\clau_0} \,\ntv{n-2}_{\nt\lxt}\,,\,
\min_{\substack{\lxt\in\clau'_0\\\lxt\neq\nt\liit}} \,\ntv{n-2}_{\nt\lxt}
\,\Big).
\end{equation}
We now apply the hypothesis of the theorem with $\clau=\clau_0$, $\clau'=\clau'_0$:\linebreak 
Either (a)~there exists $\clau_1\in\doct$ such that $\lit\in\clau_1\sbseteq\clau_0\ress{\liit}\clau'_0$,\ensep
or~(b)~$\urv_\lit > \min_{\strut\rule{0pt}{1.5ex}\lxt\,\in\,\clau_0\ress{\liit}\clau'_0\,\setminus\,\{p\}}\urv_{\nt\lxt}$.

If (b) holds, we arrive to contradiction in the following way:
\begin{equation}
\label{eq:vbiprimeter}
\ntv{n}_\lit \,=\,
\min\Big(
\min_{\substack{\lxt\in\clau_0\setminus\{q\}\\\lxt\neq\lit}} \,\ntv{n-1}_{\nt\lxt}\,,\,
\min_{\substack{\lxt\in\clau'_0\\\lxt\neq\nt\liit}} \,\ntv{n-2}_{\nt\lxt}
\,\Big)
\,\le\, 
\min_{\lxt\in\clau_0\ress{\liit}\clau'_0\setminus\{\lit\}} \,\urv_{\nt\lxt}
\,<\, \urv_\lit,
\end{equation}
where we have used the inequalities $\ntv{k}\le\urv$.

If (a) holds, we can write
\begin{IEEEeqnarray*}{rCl}
\label{eq:vbiprimeiv}
\ntv{n}_\lit \,&\le&\,
\min\Big(
\min_{\lxt\in\clau_1\cap\dot\clau_0} \,\ntv{n-1}_{\nt\lxt}\,,\,
\min_{\lxt\in\clau_1\cap\ddot\clau_0} \,\ntv{n-2}_{\nt\lxt}\,,\,
\min_{\lxt\in\clau_1\cap\clau'_0} \,\ntv{n-2}_{\nt\lxt}
\,\Big)
\\
\,&\le&\, 
\min_{\substack{\lxt\in\clau_1\\\lxt\neq\lit}} \,\ntv{n-1}_{\nt\lxt}
\,\le\, \ntv{n}_\lit,
\IEEEyesnumber
\end{IEEEeqnarray*}
where the first inequality derives from (\ref{eq:vbiprimebis}) because of the inclusion that relates $\clau_1$ to $\clau_0$ and $\clau'_0$, the second inequality uses simply that $\ntv{n-2}\le\ntv{n-1}$, and the third one follows from the definition of $\ntv{n}_p$ since $\clau_1\ni\lit$.
\ensep
Therefore, the three expressions between inequality signs in (\ref{eq:vbiprimeiv}) are in fact equal to each other,
which entails that $\clau_1$ is a realizer.
On account of the definition of $\dot\clau_1$ and the equality $\ntv{n}_\lit = \urv_\lit$,
the obtained equality entails also that
$\dot\clau_1\sbseteq\clau_1\cap\dot\clau_0\sbset\dot\clau_0$,
where the last inclusion is strict since $\liit\in\dot\clau_0$ but $\liit\notin\clau_1$.
So, $\clau_1$ is indeed a realizer with $\dot\clau_1$ strictly contained in $\dot\clau_0$.
\end{proof}

\remarks
1. Condition (b) is satisfied whenever one has $\urv_\lit > \min_{\lxt\,\in\,\clau\setminus\,\{p\}}\urv_{\nt\lxt}$, which involves neither $\clau'$ nor $\liit$.

2. The strategy of the preceding proof is akin to the SL-resolution procedure of logic programming.


\bigskip
For some doctrines whose Blake canonical form does not satisfy $\urv=\utv$
one can still see that suspicious
derivations of belief, \ie from unsatisfiable conjunctions,
are never behind the decision of accepting or rejecting certain literals.
An interesting example is what can be called the \textbf{doctrine of existence and uniqueness}:
\begin{alignat}{2}
\label{eq:eu-exists}
&\bigvee_{p\in\pset}\, p,
\\[2.5pt]
\label{eq:eu-is-unique}
&\nt p\, \lor\, \nt q,\qquad &&\text{for any two different $p,q\in\pset$.}
\end{alignat}

\pagebreak\null\vskip-13.5mm\null 

\noindent
That is, some $p\in\pset$ is true, and not two of them are true at the same time.
Notice that the Blake canonical form contains nothing more than these clauses,
since no resolution is possible.
Let us assume that a particular $p\in\pset$ is accepted, that is $\urv_p>\urv_{\nt p}$.
In contrast to what happened for the doctrine (\ref{eq:doubleopposition}), here we can see that
$\urv_p$ is always strictly larger than $\min(\urv_q,\urv_{\nt q})$
---and therefore strictly larger than $\min(\orv_q,\orv_{\nt q})$--- for any $q\in\pset$.
In fact, on the one hand, the inequality $\urv_p>\urv_{\nt p}$ obviously implies $\urv_p>\min(\urv_p,\urv_{\nt p})$.
On the other hand, for any $q\in\pset\setminus\{p\}$, the invariance of $\urv$ by the transformation $\orv\mapsto\utv$ allows to write $\urv_p>\urv_{\nt p}\ge\urv_q\ge\min(\urv_q,\urv_{\nt q})$, where the central inequality comes from the implication $q\rightarrow\nt p$, which is contained in (\ref{eq:eu-is-unique}).

\renewcommand\upla{} 

\bigskip
These arguments can be generalized into the following result:

\vskip-5mm\null 
\begin{corollary}\hskip.5em
\label{st:cor-floretamillorada}
Let us assume that the following condition is satisfied for a given $\lit\in\piset$:\ensep

\halfsmallskip
\vskip.3ex
{\parskip0pt
For~any $\clau,\clau'\in\doct$ and $\liit\in\piset\setminus\{\lit\}$ satisfying $\lit,\liit\in\clau$ and $\nt\liit\in\clau'$,
\par
one has either \textup{(a)} or \textup{(b$'$)} or both of them:
\par
\vskip.3ex
\textup{(a)}~there exists $\clau_1\in\doct$ such that $\lit\in\clau_1\sbseteq\clau\res{\liit}\clau'$;\par
\vskip.2ex
\textup{(b$'$)}~there exists $\clau_2\in\doct$ such that $\nt\lit\in\clau_2\sbseteq(\clau\setminus\{\lit\}\cup\{\nt\lit\})\res{\liit}\clau'$.
}

\noindent
In~such a situation, $\urv_\lit>\urv_{\nt\lit}$ implies $\urv_\lit=\utv_\lit$;
therefore, $\lit$ is never accepted on the basis of
unsatisfiable conjunctions.
\end{corollary}

\begin{proof}\hskip.5em
It suffices to show that condition (b$'$) allows to go from the inequality $\urv_\lit>\urv_{\nt\lit}$ to that of condition (b)~of the preceding theorem. In~fact, the invariance of $\urv$ by the transformation $\orv\mapsto\utv$ allows to write $\urv_p>\urv_{\nt p}\ge\min_{\rule{0pt}{1.5ex}\lxt\in\clau_2\setminus\{\nt p\}}\urv_{\nt\lxt}\ge\min_{\rule{0pt}{1.5ex}\lxt\,\in\,\clau\ress{\liit}\clau'\,\setminus\,\{p\}}\urv_{\nt\lxt}$.
\end{proof}

\noindent
One easily checks that the hypothesis of the preceding corollary is satisfied by any literal in the doctrine of existence and uniqueness, but not by literal $\lit$ in the doctrine (\ref{eq:doubleopposition}).

\bigskip
\textit{From now on, all revised valuations will be assumed to be computed on the basis of the Blake canonical form of the doctrine under consideration (or~any~$*$-equivalent prime conjunctive normal form),
this form being supplemented with all tertium non datur clauses}.
\ensep
Since we are dealing with prime forms,
having the equality $\urv_\lit=\utv_\lit$ guarantees that the degree of belief $\urv_\lit$ does not derive from unsatisfiable conjunctions.
When this equality holds for any initial valuation~$\orv$, we~will say that the doctrine under consideration is \dfc{unquestionable for~$\lit$.}
If this is the case for any literal $p$, then we will say simply that the doctrine is \dfc{unquestionable.}
As we have been seeing, sometimes the equality $\urv_\lit=\utv_\lit$ can fail to hold for certain valuations~$\orv$, but one can still guarantee that it holds whenever $\lit$ is accepted according to $\urv$; in that case we will say that the doctrine is \dfc{unquestionable for~$\lit$ when accepted.}

\bigskip
\xtra
The following extendability result 
can be seen as a good property of\linebreak[3] unquestionable doctrines
with respect to
the notion of definite consistency of a~partial truth assignment:

\begin{theorem}
\label{st:rich}
For an unquestionable doctrine, every definitely consistent partial truth assignment~$v$ can be extended to a consistent total truth assignment~$u$. Furthermore, for any literal $\lit$ that is undecided by $v$ one can choose $u$ so that $u_\lit = 1$.
\end{theorem}

\begin{proof}
Let us begin by extending $v$ to $w$ with the values $w_\lit = 1$ and $w_{\nt\lit} = 0$.
Let us consider now the transformed valuation $w'$.
For any $\liit$ with $w_\liit=1$ one has certainly $w'_\liit = 1$.
We claim that for such a $\liit$ one has also $w'_{\nt\liit} \le \onehalf$.
In~fact, since $w_{\nt\liit} = 1 - w_\liit = 0$, the only way to get the contrary, that is $w'_{\nt\liit} = 1$, would be having a clause $\clau$ with $\nt\liit\in\clau$ and $w_{\nt\lxt} = 1$ for any $\lxt\in\clau\setminus\{\nt\liit\}$.
Therefore, we would have $w_\lxt=0$ for any $\lxt\in\clau$.
Since $w_\lxt$ differs from $v_\lxt$ only for $\lxt=\lit,\nt\lit$, and $w_\lit=1$,
we get $v_\lxt=0$ for all $\lxt\in\clau$ except possibly one undecided literal
(namely $\nt\lit$), which contradicts the assumed definite consistency of $v$.

Let us consider now the valuation $\hat v = \mu(w')$,
where it should be recalled that $\mu(w')$ means the partial truth assignment that represents the basic decision associated with $w'$.
The preceding arguments entail that $\hat v$ is an extension of~$w$,
and therefore it is also an extension of~$v$.
On the other hand, the unquestionability of the doctrine ensures that $w' = w^*$,
and Corollary~\ref{st:dec_cor} ensures that $\hat v = \mu(w') = \mu(w^*)$ is definitely consistent.
Since $\hat v$ contains strictly less undecided literals than $v$,
and the set of literals is finite,
it is clear that iterating this procedure will eventually produce
a consistent total truth assignment.
\end{proof}


\paragraph{4.4}
\textbf{Definite Horn doctrines.}\ensep
A conjunctive normal form $\Phi(\doct)$ is said to have a \dfc{definite Horn} character when every clause $\clau\in\doct$ contains exactly one element of $\pset$. This property is easily seen to be preserved by the operations of absorption and resolution. As a consequence, it is inherited by the corresponding Blake canonical form. So, the latter has a definite Horn character \ifoi there exists a logically equivalent conjunctive normal form with the same property. In such a situation, we can say that we are dealing with a definite Horn doctrine. As we will see, for such doctrines one can arrive at consistent decisions by means of another  criterion besides the one given at the end of~\secpar{2.2}.



\medskip
For a definite Horn doctrine, the restriction of $\urv$ to $\pset$ is easily seen to depend only on the restriction of $\orv$ to $\pset$. More generally, this happens when every clause $\clau\in\doct$ contains \textit{at most} one element of $\pset$, in which case one speaks of a (simple) Horn character.
\ensep
In contrast, the properties below require a definite Horn character.

\medskip
\begin{proposition}\hskip.5em
\label{st:hornbound}
For a definite Horn doctrine one has
\begin{equation}
\label{eq:rmax}
\max_{\lit\in\pset\ustrut}\, \urv_{\nt\lit} \,\le\, \max_{\liit\in\pset\ustrut}\, \orv_{\nt\liit}.
\end{equation}
\end{proposition}
\begin{proof}\hskip.5em
The result follows from Proposition~\ref{st:uv} by taking $u$ in the following way: $u_\lit=1$ and $u_{\nt\lit}=0$ for any $\lit\in\pset$. The consistency of this truth assignment is an immediate consequence of the doctrine having a definite Horn character.
\end{proof}

\medskip
For definite Horn doctrines, one can arrive at definitely consistent decisions by considering only the restriction of $\urv$ to $\pset$. For every $\mg$ in the interval $0\le\mg\le1$ we define the \dfc{unilateral decision of margin $\mg$} in the following way: For any $\lit\in\pset$,
\begin{alignat}{3}
\label{eq:uniaccepted}
&\text{$\lit$ is accepted and $\nt\lit$ is rejected} &&\quad\text{whenever}\ \ 
\urv_\lit &&\,>\, \mg,
\\
\label{eq:unirejected}
&\text{$\lit$ is rejected and $\nt\lit$ is accepted} &&\quad\text{whenever}\ \ 
\urv_\lit &&\,<\, \mg,
\\
\label{eq:uniundecided}
&\text{$\lit$ and $\nt\lit$ are left undecided} &&\quad\text{whenever}\ \ 
\urv_\lit &&\,=\, \mg.
\end{alignat}

\smallskip
\begin{theorem}\hskip.5em
\label{st:dechorn}
For a definite Horn doctrine, and any $\mg$ in the interval $0\le\mg\le1$, the unilateral decision of margin $\mg$ associated with the upper revised valuation is always definitely consistent.
\end{theorem}
\begin{proof}\hskip.5em
According to the definition of definite consistency, we have to show that
for each $\clau\in\doct$ and every $\lit\in\clau$,\, if all $\lxt\in\clau\setminus\{\lit\}$ are rejected, then $\lit$~is accepted. To this effect, we will show that assuming $\lit$ not accepted and all~$\lxt\in\clau\setminus\{\lit\}$ rejected leads to contradiction.
Now, for a definite Horn doctrine all clauses have the form $\clau=\{\liiit\}\cup\nt{S}$, with $\liiit\in\pset$ and $S\sbset\pset$.
Since $\urv$ is ensured to satisfy the inequalities (\ref{eq:inequality}), we can write
\begin{equation}
\label{eq:ineqHorn}
\urv_\liiit \,\ge\, \min_{s\in S}\urv_s
\end{equation}
In order to arrive at the desired contradiction we will distinguish two cases:\ensep
(a)~$\lit=\liiit$;\ensep and (b)~$\lit\in\nt S$.

Case~(a): $\lit=\liiit$. In this case, $\lit=\liiit$ not accepted means that $\urv_\liiit\le\mg$,\ensep and all~$\lxt\in\clau\setminus\{\lit\}$ rejected means all $s\in S$ accepted, \ie $\urv_s>\mg$. By~plugging these inequalities in (\ref{eq:ineqHorn}) we arrive at the false conclusion that $\mg > \mg$.

Case~(b): $\lit\in\nt S$.\ensep In this case, $\liiit$ is rejected, so that $\urv_\liiit<\mg$,\ensep and the $s\in S$ are either accepted or not rejected, so that $\urv_s \ge\mg$. Therefore, we also arrive at the false conclusion that $\mg > \mg$.
\end{proof}

Compared with the bilateral decision criterion~(\ref{eq:acceptedrejected}--\ref{eq:undecided}), the unilateral one (\ref{eq:uniaccepted}--\ref{eq:uniundecided}) leaves much less room for undecidedness.
For the bilateral criterion, increasing~$\mg$ has the effect of thickening the region of undecidedness. In~contrast, for the unilateral criterion it has only the effect of moving the boundary.
This happens at the expense of disregarding the evidence in favour of $\nt\lit$ for~$\lit\in\pset$.
So, the unilateral criterion may accept $\lit$
in spite of a stronger evidence in favour of~$\nt\lit$,
or it may reject $\lit$ in spite of $\nt\lit$ having a weaker evidence than $\lit$.
On the other hand, the extent of such discrepancies is limited as stated in the following result:

\smallskip
\begin{proposition}\hskip.5em
\label{st:compdec}
The unilateral decisions associated with a definite Horn doctrine are related to the bilateral ones in the following ways:\ensep
\textup{(a)}~For any $\lit\in\pset$ satisfying $\orv_\lit+\orv_{\nt\lit}\ge1$,
if $\lit$ is accepted \textup[resp.~not rejected\textup] by the bilateral decision of margin $\mg$,
then $\lit$ is accepted \textup[resp.~not rejected\textup] by the unilateral decision of margin $(1+\mg)/2$.
\ensep
\textup{(b)}~For any $p\in\pset$, if $\lit$ is accepted \textup[resp.~not rejected\textup] by the unilateral decision of margin $\mg\ge\mgo$, then $\lit$ is also accepted \textup[resp.~not rejected\textup] by the bilateral decision of margin $\mg-\mgo$,\linebreak[3]
where $\mgo:=\max_{\liit\in\pset\ustrut} \orv_{\nt\liit}$. 
\end{proposition}
\begin{proof}\hskip.5em
In order to obtain part~(a) it suffices to notice that $\orv_\lit+\orv_{\nt\lit}\ge1$ implies $\urv_\lit+\urv_{\nt\lit}\ge1$ and therefore $\urv_\lit\,\ge\,[\,(\urv_\lit\cd-\urv_{\nt\lit})+1\,]\,/\,2$.\ensep
Part~(b) is a consequence of the inequality $\urv_\lit-\urv_{\nt\lit}\ge\urv_\lit-\mgo$, which follows immediately from~(\ref{eq:rmax}).
\end{proof}


\paragraph{4.5}
\textbf{Autarkic sets.}\ensep
A~set $\spiset\sbset\piset$ will be said to be \dfc{autarkic} for a~conjunctive normal form $\Phi(\doct)$ when it has the following properties:\linebreak 
(a)~$\spiset$ does not contain at the same time a~literal $\lit$ and its negation~$\nt\lit$;\ensep
(b)~for any clause $\clau\in\doct$, if $\clau$ contains~$\nt\lit$ for some $\lit\in\spiset$, then $\clau$ contains also some $\liit\in\spiset$.
\ensep
Such a situation arises also in satisfiability theory,
where one calls then
autarkic the partial truth assignment $u$ that sets $u_\lit=1$ and $u_{\nt\lit}=0$ for $\lit\in\Sigma$ and $u_\lit=\onehalf$ (undecided) for $\lit\in\piset\setminus(\Sigma\cup\overline\Sigma)$~\cite{ms}.

Again, the property of autarky is preserved by the operations of absorption and resolution
(checking it is a little exercise).
As a consequence,\linebreak[3] it passes on to the corresponding Blake canonical form.
So, $\spiset\sbset\piset$ is autarkic for the latter
\ifoi it is autarkic for some logically equivalent conjunctive normal form.
In such a situation, we can say simply that $\spiset $ is autarkic for the given doctrine.

\medskip
For a definite Horn doctrine, $\pset$ is easily seen to be an autarkic set.
This is a particular case of the following more general fact:
if $u$ is a truth assignment consistent with $\doct$,
then $\spiset=\{\,\lxt\in\piset\mid u_\lit=1\,\}$ is an autarkic set.
Generally speaking, however, autarkic sets need not decide on every issue.
Even so, their definition allows for the following generalization of Proposition~\ref{st:uv}:

\medskip
\begin{proposition}\hskip.5em
\label{st:detachablebound}
An~autarkic set $\spiset$ has the property that
\begin{equation}
\label{eq:smax}
\max_{\lit\in\spiset\ustrut}\, \urv_{\nt\lit} \,\le\, \max_{\liit\in\spiset\ustrut}\, \orv_{\nt\liit}.
\end{equation}
\end{proposition}
\begin{proof}\hskip.5em
Since $\urv$ is obtained by iterating the transformation $\orv\mapsto\utv$,
it~suffices to show that this transformation has the following property analogous to (\ref{eq:smax}): \,$\utv_{\nt\lit} \,\le\, \max_{\liit\in\spiset} \orv_{\nt\liit}$\, for any $\lit\in\spiset$. This follows immediately from (\ref{eq:vprime}) because, by definition, $\spiset$ being autarkic requires every clause $\clau$ that contains $\nt\lit$ with $\lit\in\spiset$ to contain also some $\liit\in\spiset$, which ensures that $\min_{\lxt\in\clau, \lxt\neq\nt\lit}\,\orv_{\nt\lxt} \le \max_{\liit\in\spiset} \orv_{\nt\liit}$.
\end{proof}

\medskip
\begin{theorem}\hskip.5em
\label{st:majority_detachable}
Assume that the original valuation satisfies
\begin{equation}
\label{eq:omaxlessthanmin}
\min_{\lit\in\spiset\ustrut}\, \orv_\lit \,-\, \max_{\lit\in\spiset\ustrut}\, \orv_{\nt\lit} \,>\, \mg,
\end{equation}
for some autarkic set $\spiset$ and some $\mg\in[0,1)$. In this case, the decision of margin $\mg$ associated with the upper revised valuation~$\urv$ accepts every proposition in $\spiset$.
\end{theorem}
\begin{proof}\hskip.5em
By combining (\ref{eq:smax}), (\ref{eq:omaxlessthanmin}) and part~(c) of Theorem~\ref{st:rev}, it follows that
\begin{equation}
\min_{\lit\in\spiset\ustrut}\, \urv_\lit \,-\, \max_{\lit\in\spiset\ustrut}\, \urv_{\nt\lit}\,>\, \mg,
\end{equation}
which implies $\urv_\lit - \urv_{\nt\lit} > \mg$ for any $\lit\in\spiset$.
\end{proof}

\medskip
\begin{theorem}\hskip.5em
\label{st:decomposition}
Assume that the original valuation satisfies
\begin{equation}
\label{eq:ounanimity}
\orv_\lit = 1,\quad \orv_{\nt\lit} = 0,\qquad\text{for any $\lit\in\spiset$,}
\end{equation}
where $\spiset$ is an autarkic set. In this case, the upper revised valuation~$\urv$ has the following properties:\ensep
\textup{(a)}~For any $\lit\in\spiset$ one has also
$\urv_\lit = 1$ and $\urv_{\nt\lit} = 0$.\ensep
\textup{(b)}~For any $\liit\in\piset$ such that neither $\liit$ nor $\nt\liit$ belongs to $\spiset$,
the upper revised value $\urv_\liit$ coincides with the value which is obtained
for a modified doctrine that leaves out any clause that contains
some element of~$\spiset\cup\nt\spiset$.
\end{theorem}

\begin{proof}\hskip.5em
Part~(a) is a direct consequence of Proposition~\ref{st:detachablebound} and part~(c) of Theorem~\ref{st:rev}.\ensep
\ensep
In order to prove part~(b) 
it~suffices to show that one can leave out the clauses that contain some element of $\spiset$;
in fact, by the definition of an autarkic set, those that contain an element of $\nt\spiset$
are ensured to contain also some element of $\spiset$.
\ensep
The statement of part~(b) will be obtained by showing that it holds for all the iterates~$\smash{\ntv{n}_\liit}$ with $n\ge1$.\ensep
In this connection, we will make use of the fact that 
$\smash{\ntv{n}_{\nt\lit} = 0}$ for any $\lit\in\spiset$ and any $n\ge0$,
which follows from part~(a) because of the inequality $\orv\le\ntv{n}\le\urv$.
Let us assume that neither $\liit$ nor $\nt\liit$ belongs to $\spiset$.
Recall that $\smash{\ntv{n}_\liit}$ is given by
\begin{equation}
\ntv{n}_\liit \,=\, \max_{\latop{\clau\in\doct}{\clau\ni\liit}}\, \min_{\latop{\lxt\in\clau}{\lxt\neq\liit}} \ntv{n-1}_{\nt\lxt}.
\end{equation}
We claim that the value that results from this formula is not altered if the $\max$ operator of the right-hand side forgets about any clause~$\clau$ that contains some $\lit$ in $\spiset$. In fact, we know that $\smash{\ntv{n-1}_{\nt\lit} = 0}$ and $\lit\not=\liit$ (because of the assumption that $\liit\not\in\spiset$). Therefore, we get
$\smash{\min_{\lxt\in\clau,\,\lxt\neq\liit}\,v_{\nt\lxt}\hskip-1.25ex{\rule[0pt]{0pt}{1.75ex}}^{(n-1)} = 0}$,
which entails that claim. 
\end{proof}


\paragraph{4.6}
\textbf{Duality.}
\ensep
The lower revised valuation $\lrv$ announced in~\secpar{2.4} is obtained by means of a dual method ---somehow it might be more appropriate to say `codual'--- which is based on the fact that the doctrine, \ie (\ref{eq:cnf}) being true, provides the following implications originated at $\lit$ (which should be compared to (\ref{eq:pimplicant})\,):
\begin{equation}
\label{eq:pimplicate}
\lit \,\rightarrow\, \bigvee_{\latop{\lxt\in\clau}{\lxt\neq\nt\lit}} \lxt,
\end{equation}
for any $\clau\in\doct$ such that $\nt\lit\in\clau$.
Such implications can be dealt with by means of a principle dual to~{\small(\teof)} that can be stated in the following way:

\newcommand\dteof{Q}
\halfsmallskip
\iim{\dteof}%
Consider an implication of the form $\lit\rightarrow \bigvee_{\lxt\in\spiset}\lxt$
with $\spiset\sbset\piset$.
As~soon as the right-hand side is not always satisfied, this implication 
restricts the degree of belief of $\lit$ to be less than or equal to the maximum degree of belief of the disjuncts~$\lxt$.
\halfsmallskip


\noindent
This leads to a downward revision of the degrees of belief according to the transformation $\orv\rightarrow\ltv$ defined by the formula
\begin{equation}
\label{eq:vprimecodual}
\ltv_\lit \,=\, \min_{\latop{\clau\in\doct}{\clau\ni\nt\lit}}\, \max_{\latop{\lxt\in\clau}{\lxt\neq\nt\lit}} \,\orv_{\lxt},
\end{equation}
whose iteration leads to the \dfc{lower revised valuation}~$\lrv$.

Equivalently, $\lrv$ is given by the formula $\lrv_\lit = 1 - {\hat\orv}^\ast_{\nt p}$,\, where ${\hat\orv}$ is related to $\orv$ by the formula ${\hat\orv}_\lit = 1 - \orv_{\nt p}$
(${\hat\orv}$ believes $\lit$ to the extent that $\orv$ does not believe $\nt\lit$).
\ensep
This characterization allows to obtain the results for $\lrv$ from those for $\urv$.

When $\orv$ is a balanced valuation, then $\lrv_\lit = 1 - \urv_{\nt p}$, so that $\lrv_\lit - \!\lrv_{\nt p} \,=\, \urv_\lit - \urv_{\nt p}$. Therefore, the upper and lower revised valuations lead then exactly to the same decisions.


\section{Application to specific domains}

In this section we apply the preceding ideas and results to three specific domains.
Unfortunately, there is no space in this article for developing other applications.

\paragraph{5.1}
\textbf{One proposition being equivalent to the conjunction of two other ones.}\hskip.75em
This is the problem with which we started the article, 
\ie $\pset=\{p,q,\thesis\}$ with the doctrine $\thesis \leftrightarrow(\lit\land\liit)$.
By rewriting the connective $\leftrightarrow$ in terms of $\land,\lor,\lnot$,  
one easily arrives at the corresponding Blake canonical form, which has a definite Horn character:
\begin{equation}
\label{eq:dpdoctrine}
(\nt p \lor \nt q \lor \thesis) \,\land\, (p \lor \nt \thesis)\,\land\, (q \lor \nt \thesis)
\end{equation}

The one-step revision transformation, which uses the clauses above and also the \textit{tertium non datur} ones, is given by
\begin{alignat}{4}
\label{eq:pp}
&v'_\lit &&=\, \max\, (\orv_\lit, \orv_\thesis), \qquad&
&v'_{\nt\lit} &&=\, \max\, (\orv_{\nt\lit}, \min(\orv_\liit, \orv_{\nt\thesis})), \\
\label{eq:qq}
&v'_\liit &&=\, \max\, (\orv_\liit, \orv_\thesis), \qquad&
&v'_{\nt\liit} &&=\, \max\, (\orv_{\nt\liit}, \min(\orv_\lit, \orv_{\nt\thesis})), \\
\label{eq:tt}
&v'_\thesis &&=\, \max\, (\orv_\thesis, \min(\orv_\lit, \orv_\liit)), \qquad&
&v'_{\nt\thesis} &&=\, \max\, (\orv_{\nt\thesis}, \orv_{\nt\lit}, \orv_{\nt\liit}).
\end{alignat}
One easily sees that the equality $\urv=\utv$ holds whenever $\orv$ satisfies the following inequalities: $\orv_t\le\orv_p$, $\orv_t\le\orv_q$, $\orv_{\nt t}\ge\orv_{\nt p}$, $\orv_{\nt t}\ge\orv_{\nt q}$. More particularly, it holds whenever $\orv$ is an aggregate of several components satisfying all of these inequalities. When the component valuations are all-or-none and non-contradictory ---\ie satisfying $\orv_\lxt+\orv_{\nt\lxt}\le1$ for any $\lxt$--- these inequalities hold whenever these 
valuations are definitely consistent. Under this assumption ---that is quite reasonable in the jury scenario of~\secpar{1.1}--- the doctrine (\ref{eq:dpdoctrine}) is therefore unquestionable in the sense of \secpar{4.3}.  
In the general case, one has only $\urv=\utv'$.

As a particular example, let us consider the following aggregate of consistent truth assignments:
\begin{align}
(\orv_\lit,\,\orv_\liit,\,\orv_\thesis;\,\orv_{\nt\lit},\,\orv_{\nt\liit},\,\orv_{\nt\thesis})
\,\,&=\,\, .30\ (1,1,1;0,0,0) \,+\, .40\ (1,0,0;0,1,1) \\
&\hskip3.5pt+\, .25\ (0,1,0;1,0,1) \,+\, .05\ (0,0,0;1,1,1) \notag\\
\,\,&=\,\, (.70,\,.55,\,.30;\,.30,\,.45,\,.70). \notag
\end{align}
As one can see, these collective degrees of belief result in the inconsistent decision of 
accepting both $\lit$ and $\liit$ but rejecting~$\thesis$. In contrast, the corresponding upper revised valuation, namely
\begin{align}
(\urv_\lit,\,\urv_\liit,\,\urv_\thesis;\,\urv_{\nt\lit},\,\urv_{\nt\liit},\,\urv_{\nt\thesis})
\,\,&=\,\, (.70,\,.55,\,.55;\,.55,\,.70,\,.70),
\end{align}
results in the consistent decision of accepting~$\lit$ and rejecting both~$\liit$ and~$\thesis$. This decision holds up to a margin of $.15$, above which all three issues are left undecided. The classical example of \secpar{1.1} results in all upper revised values being equal to exactly the same value, namely~$\frac23$, so that all three issues are then undecided even for a vanishing margin. All of this is in agreement with the concept of definite consistency defined in~\secpar{2.2}.


\paragraph{5.2}
\textbf{Equivalence relation on a set $\ist$.}\hskip.75em
Here the set of atomic propositions is
$\pset=\{\,e_{xy}\mid x,y\in\ist,\,x\neq y\,\}$, where $e_{xy}$ stands for the proposition `$x$~is equivalent to $y$'. In order to follow the standard definition of an equivalence relation, one should include also the propositions~$e_{xx}$, and constrain them to being true; however, this goes against our convention~{\small(D2)} of avoiding unit clauses, which is why the definition of $\pset$ considers only pairs $xy$ with $x\neq y$.
The doctrine requires that for any pairwise different $x,y,z\in\ist$ one must have:
$e_{xy} \land e_{yz} \rightarrow e_{xz}$ (transitivity) 
and $e_{xy} \rightarrow e_{yx}$ (symmetry).
This is equivalent to identifying $e_{xy}$ with $e_{yx}$
and adopting the following clauses:
\begin{equation}
\label{eq:transe}
\nt e_{xy} \lor \nt e_{yz} \lor e_{xz},\qquad \text{for any pairwise different $x,y,z\in\ist$}.
\end{equation}
This doctrine has a definite Horn character and has the following autarkic sets: $\pset$ itself; $\varPi_-=\{\,\nt e_{xy}\mid x,y\in\ist,\,x\neq y\,\}$; $\spiset_a=\{\,\nt e_{ax}\mid x\in\ist,\,x\neq a\,\}$ for any $a\in\ist$.

We claim that this doctrine is disjoint-resolvable in the sense defined in~\secpar{4.2}, 
and that its Blake canonical form, with the \textit{tertium non datur} clauses included, consists of all clauses of the form
\begin{equation}
\label{eq:transen}
\nt e_{x_0x_1} \lor \nt e_{x_1x_2} \lor \dots \lor \nt e_{x_{n-1}x_n} \lor e_{x_0x_n},
\end{equation}
with $n\ge1$ and all $x_i$ $(0\le i\le n)$ pairwise different (which restricts $n$ to be less than or equal to the number of elements of $\ist$). In the following we will refer to such  sequences $x_0x_1\dots x_n$ as \dfc{non-cyclic paths} from $x_0$ to $x_n$, and a generic non-cyclic path will be denoted by means of the greek letter~$\gamma$. In~order to establish the preceding claim it suffices to check that:\ensep (i)~starting from~(\ref{eq:transe}) one can arrive at~(\ref{eq:transen}) for any non-cyclic path by using only disjoint resolution;\ensep
and~(ii)~any further resolution does not add any new clause (\ie the would-be new clause is absorbed by some clause of the form~(\ref{eq:transen})\,). 
\ensep
The last statement can also be strengthened to show that condition~(a) of Thm.~\ref{st:floretamillorada} is satisfied for any of the propositions $e_{xy}$ and $\nt e_{xy}$. As a consequence, $\urv$ coincides with the result of applying once the transformation $\orv\mapsto\utv$ determined by the Blake canonical form, which makes this doctrine unquestionable in the sense of \secpar{4.3}.  

More particularly, the values of $\urv(e_{xy})$ can be obtained by considering all possible non-cyclic paths $\gamma$ from $x$ to $y$ and applying the formula
\begin{equation}
\urv(e_{xy}) \hskip.65em = \hskip.65em
\max_{\ustrut\gamma}
\hskip.65em
\min_{\ustrut0\le i<n}
\hskip.65em \orv(e_{x_ix_{i+1}}).
\label{eq:paths}
\end{equation}
These values are easily seen to satisfy
\begin{equation}
\urv(e_{xz}) \,\,\ge\,\, \min\,(\urv(e_{xy}), \urv(e_{yz}))\qquad\hbox{for any $x,y,z$.}
\label{eq:ultramin}
\end{equation}
As one can easily check,
this inequality is a necessary and sufficient condition for
the following binary relations $\releq_\mg$ ($0\le\mg\le1$) to be equivalence relations:
\begin{equation}
xy\in \releq_\mg \,\,\equiv\,\, \urv(e_{xy}) \ge \mg.
\label{eq:eleq}
\end{equation}
These facts are closely related to the definite consistency of the unilateral decisions of margin~$\mg$ considered in~\secpar{4.4}.  
In fact,
one can easily check that: $\releq_\mg$ coincides with the set of pairs $xy$ such that $e_{xy}$ is accepted by an unilateral decision of margin smaller but near enough to~$\mg$, and that this decision is definitely consistent \ifoi $\releq_\mg$ is an equivalence relation.
Obviously, the equivalence relations $\releq_\mg$ become progressively finer as $\mg$ grows (\ie $\mg<\zeta$ implies $\releq_\zeta\sbseteq \releq_\mg$ in the sense of subsets of $\ist\times\ist$).
What we are obtaining is the so-called single-link method of cluster analysis~\cite[7.3]{js}. This method is usually formulated in terms of the ``dissimilarities'' $d_{xy} = 1 - \orv(e_{xy})$ and $d^*_{xy} = 1 - \urv(e_{xy})$.
The inequality (\ref{eq:ultramin}) corresponds to~the inequality
$d^*_{xz} \le \max(d^*_{xy}, d^*_{yz})$
that defines the so-called ``ultra\-metric'' distances.
The~characterization given by~Theorem~\ref{st:char} corresponds to the well-known fact that the ultrametric distance given by the single-link method is characterized by the property of being ``subdominant'' to the original dissimilarities~$d$ \cite[8.3]{js}.
\ensep
When looking at these correspondences, one must bear in mind that in cluster analysis it is customary\, to consider the parameter $\delta=1-\mg$ rather than $\mg$.

However,
all of this is achieved at the expense of disregarding any evidence for~$\nt e_{xy}$, independently of its being stronger or weaker than the evidence for~$e_{xy}$.
In contrast, the bilateral criterion looks at the balance of these evidences.
The values of $\urv(\nt e_{xy})$, which are required to this effect, can be obtained by considering non-cyclic paths $\gamma$ from $x$ to $y$ with the special feature that exactly one of the links is marked as negative; the formula to be applied is similar to (\ref{eq:paths}) but $\orv(e_{x_ix_{i+1}})$ is replaced by $\orv(\nt e_{x_ix_{i+1}})$ for the link that is marked as negative:
\begin{equation}
\urv(\nt e_{xy}) \hskip.65em = \hskip.65em
\max_{\ustrut\gamma}
\hskip.65em
\max_{\ustrut0\le k<n}
\hskip.65em
\min_{\ustrut0\le i<n}
\hskip.65em
\begin{cases}
\orv(e_{x_ix_{i+1}}),&\text{if $i\neq k$;}
\\
\orv(\nt e_{x_ix_{i+1}}),&\text{if $i= k$.}
\end{cases}
\label{eq:negpaths}
\end{equation}
By the way, from (\ref{eq:paths}) and (\ref{eq:negpaths})
it is clear that having $\orv(\nt e_{xy})<\orv(e_{xy})$ for any two different $x,y\in\ist$ implies the same property for $\urv$, which ensures that the basic decision associated with $\urv$ accepts $e_{xy}$ for all $x,y$; this improves upon the result contained in Theorem~\ref{st:majority_detachable} for $\mg=0$.

In contrast to the unilateral criterion, the bilateral one does not result in a complete hierarchy of equivalence relations going all the way from a single class of equivalence to as many classes as objects being classified. Instead, one~obtains a hierarchy where the coarsest equivalence (corresponding to $\mg=0$) may already be made of several classes.


This provides a form of cluster analysis where
dissimilarity is not simply the lack of similarity,
but it plays its own role.
In particular, this acts against falling into the ‘stringy’ clusters typical of the unilateral single-link method.
Besides, this point of view is also especially suitable for dealing with missing data.



\paragraph{5.3}
\textbf{Total order on a set $\ist$.}\hskip.75em
Here, $\pset=\{\,p_{xy}\mid x,y\in\ist,\,x\neq y\,\}$, where $p_{xy}$~stands for the proposition `$x$~is preferred to $y$'.
The doctrine requires that for any pairwise different $x,y,z\in\ist$ one must have:
$p_{xy} \land p_{yz} \rightarrow p_{xz}$\linebreak 
(transitivity), 
$p_{xy} \rightarrow \nt p_{yx}$ (asymmetry)
and $\nt p_{xy} \rightarrow p_{yx}$ (completeness).\linebreak 
In~normal form, they read as follows:
\begin{alignat}{2}
\label{eq:transp}
&\nt p_{xy} \lor\, \nt p_{yz} \lor\, p_{xz},\qquad &&\text{for any pairwise different $x,y,z\in\ist$;} \\
\label{eq:asym}
&\nt p_{xy} \lor\, \nt p_{yx},\qquad &&\text{for any two different $x,y\in\ist$;}
\\
\label{eq:compl}
&p_{xy} \lor\, p_{yx},\qquad &&\text{for any two different $x,y\in\ist$,}
\end{alignat}
where the last one does not have a Horn character.

Similarly to the case of an equivalence relation,
this doctrine is also disjoint–resolvable
and unquestionable.
Its Blake canonical form, with the \textit{tertium non datur} clauses included,
consists of all clauses of the form 
\begin{equation}
\label{eq:transpn}
\nt p_{x_0x_1} \lor\, \nt p_{x_1x_2} \lor\, \dots\, \lor\, \nt p_{x_{n-1}x_n} \lor\, p_{x_0x_n},
\end{equation}
where $x_0x_1\dots x_n$ is a non-cyclic path,
\textit{together with}
all clauses that are obtained from (\ref{eq:transpn})\,
by replacing one or more $\nt p_{x_ix_{i+1}}$ by $p_{x_{i+1}x_i}$,
and/or replacing $p_{x_0x_n}$ by $\nt p_{x_nx_0}$
(\ie by identifying $\nt p_{xy}$ with $p_{yx}$ for any two different $x,y\in\ist$).

\xtra
Let us remark also that applying Theorem~\ref{st:rich} to this domain gives the finite case of the well-known theorem of Edward Szpilrajn about the extendability of partial orders to total orders~\cite{sz}.

In this case, our general method
corresponds essentially to the voting method introduced in~1997 by Markus Schulze 
\hbox{\brwrap{
\bibref{sc}\refco
\bibref{scbis}\refsc
\dbibref{t6}{p.\,228--232}%
}},
often called the method of paths.
More precisely, when the original valuation is balanced ---which happens when every voter expresses a comparison (a~preference or a tie) about each pair of options--- we coincide with Schulze's method except for his subsequent treatment of indecisions; when the preferential information given by the voters is not complete, this procedure does not coincide exactly with any of the variants that Schulze gives in \cite{scbis}, but it has the same spirit.\ensep
For more details, we refer the reader to~\cite{crc, cri}, where it is shown that this procedure can be extended to a continuous rating method that allows to sense the closeness of two candidates at the same time that it allows to recognise certain situations that are quite opposite to a tie.

Here we will only draw attention to significance of certain autarkic sets of this doctrine. In~fact, one can easily check the autarkic character of any set of the form $\spiset=\{\,p_{xy}\mid x\in\xst,\hskip.75pt y\in\ist\setminus\xst\,\}$, where $\xst$ is any proper subset of~$\ist$.
Applied to such sets, Theorem~\ref{st:majority_detachable}\linebreak[3] ensures that the method of paths satisfies the following majority property
\hbox{\brwrap{
\dbibref{crc}{Thm.\,10.1}\refsc
\dbibref{cri}{Thm.\,8.1}%
}}: 
If~for each member of $\xst$ and every member of $\ist\setminus\xst$ there are more than half of the individual votes where the former is preferred to the latter, then the resulting social ranking also prefers each member of~$\xst$ to every member of~$\ist\setminus\xst$.

\section{Discussion and interpretation of the results}

\paragraph{6.1}
The variant that we have chosen as the main one is crucially based upon the principle {\small(\teof)} stated in \secpar{3.1}:\linebreak[3]
An~implication of the form $\lit\leftarrow \bigwedge_{\lxt\in\spiset}\lxt$ with
a satisfiable right-hand side
gives to $\lit$ at least the same degree of belief as the weakest of the conjuncts~$\lxt$.
This principle goes back to ancient philosophy,%
where it was stated by saying that \textit{peiorem semper conclusio sequitur partem}.
In more recent times, this idea 
has been brought back by several authors in connection with
different theories of degrees of belief.
For a recent overview of the subject, we refer to \cite{huberbook} and the articles therein.

Assuming that the doctrine is specified by means of a prime conjunctive normal form,
the above-mentioned principle allows to replace any given belief valuation $\orv$ 
by at least the~$\utv$ given by the max-min formula~(\ref{eq:vprime}).
In this connection, it~must be emphasized that the $\max$ operator of~(\ref{eq:vprime}) does not hinge on the dual principle {\small(\dteof)} of \secpar{4.6}.  
Instead, it appears simply as a result of having several implications leading to $\lit$:
being greater than or equal to several values certainly implies being greater than or equal to the greatest of them.

The valuation $\utv$ obtained in this way is greater than or equal to the original one $\orv$ because we have systematically included the implications $\lit\leftarrow\lit$ (through the \textit{tertium non datur} clauses $\lit\lor\nt\lit$). 
\ensep
By iterating the transformation $\orv\mapsto\utv$ we arrive at the invariant valuation~$\urv$ that we call upper revised valuation. 
For a \textit{fixed} prime conjunctive normal form, $\urv$ is characterized as the lowest of the valuations $\val$ that lie above $\orv$ and satisfy the invariance equation $\val'=\val$ (Theorem~\ref{st:char}). On~the other hand, when we consider different prime conjunctive normal forms, all~of them logically equivalent to a given doctrine, then $\urv$ is greatest when we take the Blake canonical form, that is, the prime conjunctive normal form composed of all the prime implicates
(Proposition~\ref{st:bcfmax}).
\ensep
The fact that the revised valuation $\urv$ satisfies the equality $\urv{}'=\urv$
can be seen as a general form of consistency with the doctrine; in particular, it
entails that the decisions based upon the differences $\urv_\lit-\urv_{\nt\lit}$ are always
definitely 
consistent (Corollary~\ref{st:dec_cor}). 

The dual variant works in a similar way, but the non-decreasing transformation $\orv\mapsto\utv$ is replaced by a non-increasing one $\orv\mapsto\ltv$. The resulting lower revised valuation~$\lrv$ is characterized as the greatest of the valuations $\val$ that lie below~$\orv$ and satisfy the invariance equation ${}^\prime\kern-.25pt\val=\val$. Somehow, it would be appropriate for a skeptic believer, whereas the upper revised valuation would correspond to an easy believer.


Many of the above underlying ideas can be found somewhere in the literature.
Among the works closest to ours we can mention that of Nicholas Rescher about plausible reasoning~\cite{rescher76}.
However, and making abstraction of certain differences in the setup,
that work can be seen as starting from a valuation $\val$ that already satisfies the equality $\val'=\val$ 
(which follows from \cite[p.\,15, {\small(P4)}]{rescher76} because of the \emph{tertium non datur} clauses);
even so, the consistency of the decisions based upon the differences
$\val_\lit - \val_{\nt\lit} = \urv_\lit-\urv_{\nt\lit}$ is not obtained as a theorem, but it forms part of an axiom~\cite[p.\,16, {\small(P6)}]{rescher76}.
On the other hand, the transformation $\val\rightarrow\val'$ is still used ---\cite[p.\,19]{rescher76}--- but only as a means for extending the initial valuation to any compound proposition with a zero initial value (which extension is done in a single step).


\renewcommand\upla{\vskip1pt}

\upla
In contrast to Rescher and other authors, the valuations considered in the present work are defined only for the members of $\piset$, \ie the basic propositions and their negations. In so doing, we take the view that any issue of interest is included in $\piset$, as a pair formed by an atom and its negation, and that its logical connection to the other issues is specified by suitable clauses in the doctrine.

\upla
Since it involves only the $\max$ and $\min$ operators,
the transformation $\orv\rightarrow\utv$,
and therefore also the transformation $\orv\rightarrow\urv$,
have a purely ordinal character:
The ordering of $\piset$ by $\urv$ depends only on its ordering by~$\orv$. 
In~particular, the basic decision associated with $\urv$
is based wholly on comparisons.
However, one cannot say the same about the decisions that require a certain positive margin.
Such decisions make sense only for valuations that have a cardinal character,
as in the case of judgment aggregation, where $\orv_\lit$ means the fraction of people
who consider $\lit$ true. 
Another result whose meaningfulness requires cardinal valuations 
is the property of continuity stated in part~(a) of Theorem~\ref{st:rev}.
By the way, this property ensures that the accepted propositions remain accepted
when the valuation $\orv$ undergoes slight variations.

\upla
Besides its making sense in the aggregation of individual judgments,\linebreak 
belief adjustment may already be taking place to some extent within\linebreak 
every individual.
This is somehow unavoidable
if the individuals are required to produce consistent judgments,
as it is usually the case.
In particular, this means that the individuals are already allowing some issues
to influence the others.
In~view of this, it is quite reasonable to dispense the aggregation method
from complying with the condition of issue-by-issue aggregation
considered in \secpar{1.3}.

\upla
Our method disregards the condition of issue-by-issue aggregation
but complies with the following condition of respect for consistent majority decisions:
If the majority criterion applied to the original valuation $\orv$ decides on each issue and is consistent with the doctrine, then this decision is respected
(Theorem~\ref{st:majority}). 
Majority is here understood as having $\orv_\lit > \onehalf > \orv_{\nt\lit}$
(rather than simply $\orv_\lit > \orv_{\nt\lit}$).
Furthermore, we have also a property of respect for unanimity
(Theorem~\ref{st:unanimity}): 
if~a~particular prop\-osition is accepted by every individual,
and the individual judgments are consistent,
then that proposition is also accepted by the collective judgment.
Another good property is the monotonicity given by Theorem~\ref{st:monoThm}.


\pagebreak\null\vskip-19mm\null 

\paragraph{6.2}
The fractional degrees of belief form a continuum of possibilities that stretches over the all-or-none framework of classical logic.
This allows for $\lit$ and $\nt\lit$ not being exactly the semantic negation of each other, but rather the opposite, or antithesis, of each other.
One may object that this is not compatible with the excluded-middle principle~$\lit\lor\nt\lit$.
However, this principle somehow loses its character just as fractional valuations come in.
In~fact, its role in connection with the latter is only through the excluded-middle clauses that we systematically include in the Blake canonical form; and this has only the following two effects:
(a)~providing the trivial implications $\lit\rightarrow\lit$ and $\nt\lit\rightarrow\nt\lit$,
through which the upper revised degrees of belief become larger than or equal to the original ones;
(b)~forbidding 
any implication of the form $\nt\lit\land\lit\land\chi\rightarrow\thesis$,
which would be a gratuitous source of belief.
Effect~(b) occurs because clauses are restricted to be prime, which prevents them from containing $\lit\lor\nt\lit$.

When $\nt\lit$ is the opposite of $\lit$ rather than its denial, 
then intermediate degrees of belief can somehow be identified 
with belief in intermediate possibilities of fact. 
For~instance, if $\lit$ means `white' and $\nt\lit$ means `black'\linebreak[3]
---and even more if $\lit$ means `rather white' and $\nt\lit$ means `rather black'---
then $(\val_\lit,\val_{\nt\lit})=(0.3,0.7)$ 
can be interpreted as giving belief to the gray that combines white and black in the given proportions.
When the valuation is not balanced, then we are adding a sort of intensity of belief that can go from a full lack of opinion to a self-contradictory one.

\renewcommand\uplapar{\vskip-8mm\null}

\uplapar
\paragraph{6.3}
Let us look back on the special problem of collegial courts with which we started the article.
As we did in \secpar{1.1} and \secpar{5.1}, we take as archetype the doctrine
$\thesis\leftrightarrow\lit\land\liit$, where $\thesis$ means `being guilty' of a certain offence.

The data are the fractions of the jury who adhere to each of the propositions in question and their respective negations.
These numbers can certainly be viewed as degrees of collective belief.
The problem is that the decision that is naturally associated to this valuation,  
namely accepting~$\lxt$, and rejecting~$\nt\lxt$, whenever $\orv_\lxt > \orv_{\nt\lxt}$,
may be inconsistent with the doctrine.
This~may happen even when each member of the jury is expressing a consistent opinion.
Thus, in the particular case of the doctrine $\thesis\leftrightarrow\lit\land\liit$
we have seen examples of consistent individual judgments that result in
having at the same time
$\orv_\lit>\orv_{\nt\lit}$ and $\orv_\liit>\orv_{\nt\liit}$
(guilty by the standard premise-based criterion)
but $\orv_\thesis<\orv_{\nt\thesis}$
(not~guilty by the standard conclusion-based criterion).

In contrast, the revised valuation $\urv$ has the property that the associated decision is always definitely consistent with the doctrine in the sense defined in~\secpar{2.2}.
In particular, for the doctrine $\thesis\leftrightarrow\lit\land\liit$
one is ensured to have\linebreak 
$\urv_\lit>\urv_{\nt\lit}$ and $\urv_\liit>\urv_{\nt\liit}$
(guilty by the revised premise-based criterion)
\,if\,and only\,if\, $\urv_\thesis>\urv_{\nt\thesis}$
(guilty by the revised conclusion-based criterion).

Of course, it may well happen that $\urv_\lxt = \urv_{\nt\lxt}$, which does not allow to decide between $\lxt$ and $\nt\lxt$. Such equalities are somehow easier to happen than the analogous ones for the original valuation $\orv$.
However, they can be ruled out in the case of a unanimous consistent belief in $\lxt$ (Theorem~\ref{st:unanimity}), and also, by continuity, 
if we are near enough to such unanimity.

The process that leads from the original degrees of belief $\orv$ to the revised ones $\urv$ can be seen as a quantitative virtual deliberation in accordance with the implications contained in the doctrine and with principle~{\small(\teof)}.

The suitability of a process of this kind in the context of law courts
was advocated by L.~Jonathan Cohen
in his celebrated book \textit{The~Probable and the Provable}~\cite{coh}.

For the doctrine $\thesis\leftrightarrow\lit\land\liit$, the process of deliberation goes not only from $\lit$ and $\liit$ to $\thesis$, but also the other way
(which makes it somewhat inappropriate to call $\lit$ and $\liit$ the `premises' and $\thesis$ the `conclusion').
Letting the implications $\thesis\rightarrow\lit$ and $\thesis\rightarrow\liit$ to come in\, conforms to the point of view expressed by Kornhauser and Sager in \cite{ks04}:
``In actual deliberation, our commitments to outcomes may sometimes be more basic and fundamental than our commitment to the `principles' or `reasons' that ostensibly support them.''

The existence of such a direct belief about $\thesis$ makes $\lit$ and $\liit$ not independent from each other. Such a situation could be avoided by submitting $\lit$~and~$\liit$ to the consideration of two separate juries (so that $\orv_\thesis = \orv_{\nt\thesis} = 0$). In~such a setting,
if~both juries have the same number of members and none of the jurors abstains,
then the basic decision according to $\urv$ can be seen to coincide with the standard premise-based one.

Another aspect where our method matches the standard principles of law
is the fact that the decision about $\thesis$ is obtained from a balance
between the arguments for $\thesis$ and those for~$\nt\thesis$.
In~fact, this is the main idea of the adversarial system of justice that operates in most jurisdictions.
More specifically, our basic decision criterion, \ie that of margin $0$,
would correspond to the notion of ``preponderance of evidence'',
also known as ``balance of probabilities'', that defines the standard of proof usually adopted in civil cases.
In~contrast, the standard of proof ``beyond a reasonable doubt'' 
typical of criminal cases 
would correspond to a (perhaps unilateral) decision criterion of margin~$\mg$, with $\mg$ near enough to~$1$.





\vskip-8mm\null 

\end{document}